\documentclass{article}

\usepackage{microtype}
\usepackage{graphicx}
\usepackage{subfigure}
\usepackage{booktabs} 
\usepackage{natbib}
\usepackage{amsmath}
\usepackage{amsfonts}
\usepackage{amssymb}
\usepackage{amsthm}
\usepackage{mathtools}
\usepackage[pdf]{pstricks}
\usepackage{epsfig}
\usepackage{pst-grad} 
\usepackage{pst-plot} 
\usepackage{comment}

\usepackage[noend]{algcompatible}
\usepackage{algorithm}

\usepackage{hyperref}

\renewcommand{\algorithmiccomment}[1]{\hfill \small//~#1\normalsize}

\usepackage[accepted]{icml2018}

\icmltitlerunning{Transfer in Deep Reinforcement Learning
Using Successor Features and Generalised Policy Improvement}

\begin{document}

\twocolumn[
\icmltitle{Transfer in Deep Reinforcement Learning
Using \\ Successor Features and Generalised Policy Improvement}




\begin{icmlauthorlist}
\icmlauthor{Andr\'e Barreto}{dm}
\icmlauthor{Diana Borsa}{dm}
\icmlauthor{John Quan}{dm}
\icmlauthor{Tom Schaul}{dm}
\icmlauthor{David Silver}{dm} \\
\icmlauthor{Matteo Hessel}{dm}
\icmlauthor{Daniel Mankowitz}{dm}
\icmlauthor{Augustin \v{Z}\'idek }{dm}
\icmlauthor{R\'emi Munos}{dm}
\end{icmlauthorlist}

\icmlaffiliation{dm}{DeepMind}
\icmlcorrespondingauthor{Andr\'e Barreto}{andrebarreto@google.com}

\icmlkeywords{Machine Learning, ICML}

\vskip 0.3in
]



\printAffiliationsAndNotice{}  

\newcommand{\cp}{\citep}
\newcommand{\ca}{\citeauthor}
\newcommand{\cy}{\citeyear}
\newcommand{\ct}{\citet}
\newcommand{\ctp}[1]{\ca{#1}'s~(\cy{#1})}
\newcommand{\cwp}[1]{\ca{#1}, \cy{#1}}
\newcommand{\ctt}[1]{\ca{#1}~\cp{#1}} 

\newcommand{\mat}[1]{\ensuremath{\boldsymbol{\mathrm{#1}}}}
\newcommand{\fnc}[1]{\ensuremath{\mathrm{#1}}}

\newtheorem{theorem}{Theorem}
\newtheorem{lemma}{Lemma}
\newtheorem{proposition}{Proposition}

\newcommand{\R}{\ensuremath{\mathbb{R}}}
\newcommand{\M}{\ensuremath{\mathcal{M}^{\phi}}}
\newcommand{\MM}{\ensuremath{\mathcal{M}}} 
\newcommand{\E}{\ensuremath{\mathrm{E}}}
\newcommand{\MMtr}{\ensuremath{{\hat{\MM}}}} 
\newcommand{\Mtr}{\ensuremath{M}}
\newcommand{\wtr}{\ensuremath{{\mat{w}}}}
\newcommand{\rtr}{\ensuremath{{r}}}
\newcommand{\rvtr}{\ensuremath{{\mat{r}}}}
\newcommand{\Rtr}{\ensuremath{{R}}}
\newcommand{\Wtr}{\ensuremath{{\mat{W}}}}

\newcommand{\Normal}{\ensuremath{\mathcal{N}}}

\newcommand{\w}{\mat{w}}
\newcommand{\WW}{\mat{W}}
\newcommand{\Wt}{\ensuremath{\tilde{\mat{W}}}}
\renewcommand{\r}{\mat{r}}
\newcommand{\z}{\mat{z}}
\newcommand{\Z}{\mat{Z}}
\newcommand{\vphi}{\mat{\phi}}
\newcommand{\vpsi}{\mat{\psi}}
\newcommand{\vprm}{\mat{\z}}
\newcommand{\tvphi}{\ensuremath{\tilde{\mat{\phi}}}}
\newcommand{\tvpsi}{\ensuremath{\tilde{\mat{\psi}} }}
\newcommand{\tPsi}{\ensuremath{\tilde{\Psi} }}
\newcommand{\tmpsi}{\ensuremath{\tilde{\mat{\Psi}} }}
\newcommand{\MPhi}{\mat{\Phi}}
\newcommand{\MPsi}{\mat{\Psi}}
\renewcommand{\P}{\mat{P}}
\renewcommand{\t}{\ensuremath{\top}}

\renewcommand{\th}[1]{\text{#1$^\mathrm{th}$}}
\newcommand{\ith}{\th{i}}
\newcommand{\jth}{\th{j}}
\newcommand{\kth}{\th{k}}

\newcommand{\argmin}{\ensuremath{\mathrm{argmin}}}
\newcommand{\argmax}{\ensuremath{\mathrm{argmax}}}
\renewcommand{\min}{\ensuremath{\mathrm{min}}}

\newcommand{\Qmax}{\ensuremath{Q_{\mathrm{max}}}}
\newcommand{\pimax}{\ensuremath{\pi_{\max}}}
\newcommand{\Qpimax}{\ensuremath{Q^{\pi_{\max}}}}
\newcommand{\Tpimax}{\ensuremath{T^{\pi_{\max}}}}

\newcommand{\Mh}{\ensuremath{\hat{M}}} 
\newcommand{\Qh}{\ensuremath{\hat{Q}}}
\newcommand{\pih}{\ensuremath{\hat{\pi}}}
\newcommand{\wh}{\ensuremath{\hat{\w}}}
\newcommand{\Qhmax}{\ensuremath{\hat{Q}_{\mathrm{max}}}}
\newcommand{\pihmax}{\ensuremath{\pi}}
\newcommand{\Qhpimax}{\ensuremath{\hat{Q}^{\pi_{\max}}}}
\newcommand{\Thpimax}{\ensuremath{\hat{T}^{\pi_{\max}}}}

\newcommand{\rdif}{\ensuremath{\delta}}

\newcommand{\Qt}{\ensuremath{\tilde{Q}}}
\newcommand{\vQt}{\ensuremath{\tilde{\mat{Q}}}}
\newcommand{\pit}{\ensuremath{\tilde{\pi}}}
\newcommand{\Qtmax}{\ensuremath{\tilde{Q}_{\mathrm{max}}}}
\newcommand{\pitmax}{\ensuremath{{\pi}}}
\newcommand{\Qpitmax}{\ensuremath{Q^{\pi}}}
\newcommand{\Tpitmax}{\ensuremath{T^{\pi}}}
\newcommand{\wt}{\ensuremath{\tilde{\w}}}
\renewcommand{\H}{\mat{H}} 

\newcommand{\agt}{\fnc{agt}}

\newcommand{\piexp}{\ensuremath{\pi}}
\newcommand{\piexpj}{\ensuremath{_{\piexp_j}}}
\newcommand{\piexpi}{\ensuremath{_{\piexp_i}}}
\newcommand{\piexpk}{\ensuremath{_{\piexp_k}}}

\newcommand{\la}{\ensuremath{\leftarrow}}

\newcommand{\vpi}{\ensuremath{\mat{v}^{\pi}}}
\newcommand{\rpi}{\ensuremath{\mat{r}^{\pi}}}
\newcommand{\Ppi}{\ensuremath{\mat{P}^{\pi}}}

\newcommand{\phimax}{\ensuremath{\vphi_{\max}}}

\renewcommand{\S}{\ensuremath{\mathcal{S}}}
\newcommand{\A}{\ensuremath{\mathcal{A}}}
\newcommand{\W}{\ensuremath{\mathcal{W}}}

\newcommand{\score}{\ensuremath{\mathrm{score}}}
\newcommand{\currscore}{\ensuremath{\mathrm{curr\_score}}}
\newcommand{\ntpu}{\ensuremath{\mathrm{used}}}
\newcommand{\act}{\ensuremath{\mathrm{c}}}

\newcommand{\I}[1]{\ensuremath{\delta\{#1\}}} 
\newcommand{\II}{\ensuremath{\delta}} 

\newcommand{\vf}{\ensuremath{\mat{\varphi}}}
\newcommand{\vfi}{\ensuremath{\mat{\varphi}_i}}
\newcommand{\vfp}{\ensuremath{\mat{\varphi}_p}}
\newcommand{\vweights}{\mat{z}}
\newcommand{\od}{\ensuremath{\mat{o}}}
\newcommand{\rank}{\ensuremath{\mathrm{rank}}}

\newcommand{\ns}{\ensuremath{\mathrm{ns}}}
\newcommand{\nstr}{\ensuremath{n_{\mathrm{tr}}}}
\newcommand{\nsts}{\ensuremath{n_{\mathrm{ts}}}}
\newcommand{\dist}{\ensuremath{\mathcal{D}}}
\newcommand{\rt}{\ensuremath{\tilde{r}}}
\newcommand{\rvt}{\ensuremath{\tilde{\mat{r}}}}
\newcommand{\vrt}{\ensuremath{\tilde{\mat{r}}}}
\newcommand{\params}{\ensuremath{\mat{\theta}}}

\newcommand{\eb}{\ensuremath{\mathrm{extend\_basis}}}
\newcommand{\sfgpi}{SF\hspace{0.2mm}\&\hspace{0.2mm}GPI}
\newcommand{\sfgpirl}{\sfgpi-continual}
\newcommand{\sfgpisl}{\sfgpi-transfer}

\newcommand{\sm}{\text{-}}
\newcommand{\mo}{{\normalfont-{\tt 1}}}
\newcommand{\taskfive}{{\tt 1100}}
\newcommand{\tasksix}{{\tt 0111}}
\newcommand{\taskseven}{{\tt 1111}}
\newcommand{\taskeight}{{\tt \mo000}}
\newcommand{\tasknine}{{\tt \mo\mo00}}
\newcommand{\taskten}{{\tt \mo100}}
\newcommand{\taskeleven}{{\tt \mo1\mo0}}
\newcommand{\tasktwelve}{{\tt \mo101}}
\newcommand{\taskthirteen}{{\tt \mo1\mo1}}
\newcommand{\taskexample}{{\tt 1\mo00}}

\newcommand{\pitest}{\ensuremath{\pi_{\mathrm{test}}}}

\begin{abstract}
The ability to transfer skills across tasks has the potential to scale up reinforcement learning (RL) agents to environments currently out of reach. Recently, a framework based on two ideas, successor features (SFs) and generalised policy improvement (GPI), has been introduced as a principled way of transferring skills. In this paper we extend the \sfgpi\ framework in two ways. One of the basic assumptions underlying the original formulation of \sfgpi\ is that rewards for all tasks of interest can be computed as linear combinations of a fixed set of features. We relax this constraint and show that the theoretical guarantees supporting the framework can be extended to any set of tasks that only differ in the reward function. Our second contribution is to show that one can use the reward functions themselves as features for future tasks, without any loss of expressiveness, thus removing the need to specify a set of features beforehand. This makes it possible to combine \sfgpi\ with deep learning in a more stable way. We empirically verify this claim on a complex 3D environment where observations are images from a first-person perspective. We show that the transfer promoted by \sfgpi\ leads to very good policies on unseen tasks almost instantaneously. We also describe how to learn policies specialised to the new tasks in a way that allows them to be added to the agent's set of skills, and thus be reused in the future.
\end{abstract}

\section{Introduction}
\label{sec:introduction}

In recent years reinforcement learning (RL) has undergone a major change in terms of the scale of its applications: from relatively small and well-controlled benchmarks to problems designed to challenge humans---who are now consistently outperformed by artificial agents in domains considered out of reach only a few years ago~\cp{mnih2015human,bowling2015headsup,silver2016mastering,silver2017mastering}. 

At the core of this shift has been deep learning, a machine learning paradigm that recently attracted a lot of attention due to impressive accomplishments across many areas~\cp{goodfellow2016deep}. Despite the successes achieved by the combination of deep learning and RL, dubbed ``deep RL'', the agents's basic mechanics have remained essentially the same, with each problem tackled as an isolated monolithic challenge. An alternative would be for our agents to decompose a problem into smaller sub-problems, or ``tasks'', whose solutions can be reused multiple times, in different scenarios. This ability of explicitly transferring skills to quickly adapt to new tasks could lead to another leap in the scale of RL applications.

Recently, \ct{barreto2017successor} proposed a framework for transfer based on two ideas: generalised policy improvement (GPI), a generalisation of the classic dynamic-programming operation, and successor features (SFs), a representation scheme that makes it possible to quickly evaluate a policy across many tasks. The \sfgpi\ approach is appealing because it allows transfer to take place between any two tasks, regardless of their temporal order, and it integrates almost seamlessly into the RL framework.

In this paper we extend \ctp{barreto2017successor} framework in two ways. First, we argue that its applicability is broader than initially shown. \sfgpi\ was designed for the scenario where each task corresponds to a different reward function; one of the basic assumptions in the original formulation was that the rewards of all tasks can be computed as a linear combination of a fixed set of features. We show that such an assumption is not strictly necessary, and in fact it is possible to have guarantees on the performance of the transferred policy even on tasks that are not in the span of the features.  

The realisation above adds some flexibility to the problem of computing features that are useful for transfer. Our second contribution is to show a simple way of addressing this problem that can be easily combined with deep learning. Specifically, by looking at the associated approximation from a slightly different angle, we show that one can replace the features with actual rewards. This makes it possible to apply \sfgpi\ online at scale. In order to verify this claim, we revisit one of \ctp{barreto2017successor} experiments in a much more challenging format, replacing a fully observable $2$-dimensional environment with a $3$-dimensional domain where observations are images from a first-person perspective. We show that the transfer promoted by \sfgpi\ leads to good policies on unseen tasks almost instantaneously. Furthermore, we show how to learn policies that are specialised to the new tasks in a way that allows them to be added to the agent's ever-growing set of skills, a crucial ability for continual learning~\cp{thrun96learning}.

\section{Background}
\label{sec:background}

In this section we present the background material that will serve as a foundation for the rest of the paper.

\subsection{Reinforcement learning}
\label{sec:rl}

In RL an agent interacts with an environment and selects actions in order to maximise
the expected amount of reward received~\cp{sutton98reinforcement}. 
We model this scenario using the formalism of 
\emph{Markov decision processes} (MDPs, 
\cwp{puterman94markov}). 
An MDP is a tuple $M \equiv (\S,\A,p,R,\gamma)$
whose components are defined as follows.
The sets $\S$ and $\A$ are the state and action spaces, respectively.
The function $p$ defines the \emph{dynamics} of the MDP: 
specifically, $p(\cdot|s,a)$ gives the next-state 
distribution upon taking action $a$ in state $s$. 
The random variable $R(s,a,s')$ determines the reward received in the transition $s \xrightarrow{a} s'$; it is often convenient to consider 
the expected value of this variable, $r(s,a,s')$.
Finally, the discount factor $\gamma \in [0,1)$ gives 
smaller weights to future rewards. 
Given an MDP, the goal is to maximise the expected 
\emph{return}
$
G_{t} = \sum_{i=0}^{\infty} \gamma^{i} R_{t+i},
$
where $R_t = R(S_t, A_t, S_{t+1})$.
In order to do so the agent computes a \emph{policy} $\pi: \S \mapsto \A$. 

A principled way to address the RL problem is to use methods
derived from \emph{dynamic programming} (DP)~\cp{puterman94markov}. 
RL methods based on DP usually compute the \emph{action-value function} of a policy $\pi$,
defined as:
\begin{equation}
\label{eq:Q}
Q^{\pi}(s,a) \equiv \E^{\pi} \left[ G_{t} \,|\, S_{t} = s, A_{t} = a \right],
\end{equation}
where $\E^{\pi}[\cdot]$ denotes expectation over the sequences of transitions induced by $\pi$. 
The computation of $Q^{\pi}(s,a)$ is called {\em policy evaluation}. 
Once a policy $\pi$ has been evaluated, we can compute a \emph{greedy} policy $\pi'(s) \in \argmax_{a} Q^{\pi}(s,a)$ that is guaranteed to perform at least as well as $\pi$, that is: $Q^{\pi'}(s,a) \ge Q^{\pi}(s,a)$ for any $(s,a) \in \S \times \A$. The computation of $\pi'$ is referred to as \emph{policy improvement}. The alternation between policy evaluation and policy improvement is at the core of many DP-based RL algorithms, which usually carry out these steps only approximately.

As a convention, we will add a tilde to a symbol to indicate that the associated quantity is an approximation; we will then refer to the respective tunable parameters as \params. For example, the agent computes an approximation $\Qt^\pi \approx Q^\pi$ by tuning $\params_Q$. In deep RL some of the approximations computed by the agent, like $\Qt^{\pi}$, are represented by complex nonlinear approximators composed of many levels of nested tunable functions; among these, the most popular models are by far deep neural networks~\cp{goodfellow2016deep}. 

In this paper we are interested in the problem of \emph{transfer} in RL~\cp{taylor2009transfer,lazaric2012transfer}. Specifically, we ask the question: given a set of MDPs that only differ in their reward function, how can we leverage knowledge gained in some MDPs to speed up the solution of others?

\subsection{\sfgpi}
\label{sec:sfs_gpi}

\ct{barreto2017successor} propose a framework to solve a restricted version of the problem above. Specifically, they restrict the scenario of interest to MDPs whose expected one-step reward can be written as
\begin{equation}
\label{eq:reward}
r(s,a, s') = \vphi(s,a, s')^\t\w,
\end{equation}
where $\vphi(s,a, s') \in \R^{d}$ are features of $(s,a,s')$ and 
$\w \in \R^{d}$ are weights. In order to build some intuition, it helps to think of $\vphi(s,a,s')$ as salient events that may be desirable or undesirable to the agent. Based on~(\ref{eq:reward}) one can define an environment $\M(\S, \A, p, \gamma)$ as
{\small
\begin{equation}
\label{eq:M}
\M \equiv \{ M(\S, \A, p, r, \gamma) | r(s,a,s') = \vphi(s,a,s')^\t\w \},
\end{equation}
}
\hspace{-2mm} that is, \M\ is the set of MDPs induced by $\vphi$ through 
all possible instantiations of \w.
We call each $M \in \M$ a \emph{task}.
Given a task $M_i \in \M$ defined by $\w_i \in \R^d$, 
we will use $Q^{\pi}_i$ to refer to the value function of $\pi$ 
on $M_i$.

\ct{barreto2017successor} propose \sfgpi\ as a way to promote transfer between tasks in \M. As the name suggests, GPI is a generalisation of the policy improvement step described in Section~\ref{sec:rl}. The difference is that in GPI the improved policy is computed based on a \emph{set} of value functions rather than on a single one. 
Let $Q^{\pi_1}, Q^{\pi_2}, ... Q^{\pi_n}$ be the action-value functions of $n$ policies defined on a given MDP, and let $Q^{\max} = \max_i Q^{\pi_i}$. If we define $\pi(s) \la \argmax_a Q^{\max}(s,a)$ for all $s \in \S$, then $Q^{\pi}(s,a) \ge Q^{\max}(s,a)$ for all $(s,a) \in \S \times \A$. The result also extends to the scenario where we replace $Q^{\pi_i}$ with approximations $\Qt^{\pi_i}$, in which case the lower bound on $Q^{\pi}(s,a)$ gets looser with the approximation error, as in approximate DP~\cp{bertsekas96neuro-dynamic}.

In the context of transfer, GPI makes it possible to leverage  
knowledge accumulated over time, across multiple tasks, to learn a new task faster.
Suppose that the agent has access to $n$ policies $\pi_1, \pi_2, ..., \pi_n$. These can be arbitrary policies, but for the sake of the argument let us assume they are solutions for tasks $M_1$, $M_2$, ..., $M_n$. Suppose also that when exposed to a new task $M_{n+1} \in \M$ the agent computes $Q^{\piexpi}_{n+1}$---the value functions of the policies $\piexp_i$ under the reward function induced by $\w_{n+1}$.
In this case, applying GPI to the set $\{Q^{\piexp_1}_{n+1}, Q^{\piexp_2}_{n+1},..., Q^{\piexp_n}_{n+1}\}$ will result in a policy that performs at least as well as any of the policies $\piexp_i$.

Clearly, the approach above is appealing only if we have a way to quickly compute the value functions of the policies $\pi_i$ on the task $M_{n+1}$. This is where SFs come in handy. SFs make it possible to compute the value of a policy $\pi$ on \emph{any} task $M_i \in \M$ by simply plugging in the representation the vector $\w_i$ defining the task. Specifically, if we substitute~(\ref{eq:reward}) in~(\ref{eq:Q}) we have
 \begin{align}
  \label{eq:sf} 
  \nonumber Q^{\pi}_{i}(s,a) 
  \nonumber & = \E^{\pi} \left[{\textstyle \sum_{i=t}^{\infty}} \gamma^{i-t} 
\vphi_{i+1} \,|\, S_t = s, A_t = a \right]^{\t} \w_{i} \\
  & \equiv \vpsi^{\pi}(s,a)^{\t} \w_{i},
 \end{align}
where $\vphi_t = \vphi(s_t, a_t, s_{t+1})$ and $\vpsi^{\pi}(s,a)$ are the SFs of $(s,a)$ under policy $\pi$. As one can see, SFs decouple the dynamics of the MDP $M_i$ from its rewards~\cp{dayan93improving}. One benefit of doing so is that if we replace $\w_i$ with $\w_j$ in ~(\ref{eq:sf}) we immediately obtain the evaluation of $\pi$ on task $M_j$. 
It is also easy to show that 
{\small
\begin{align}
 \label{eq:bellman_psi}
  \vpsi^\pi(s, a) 
  & = \E^{\pi}[\vphi_{t+1}+ \gamma \vpsi^\pi(S_{t+1}, \pi(S_{t+1}))  \,|\, S_t 
= s, A_t = a ],
 \end{align}
}
\hspace{-2mm} that is, SFs satisfy a Bellman equation in which 
$\phi_i$ play the role of rewards. 
Therefore, SFs can be learned using any conventional RL method~\cp{szepesvari2010algorithms}.

The combination of SFs and GPI provides a general framework for transfer in environments of the form~(\ref{eq:M}). Suppose that we have learned the functions $Q^{\piexpi}_i$ using the representation scheme~(\ref{eq:sf}). When exposed to the task defined by $r_{n+1}(s,a,s') = \vphi(s,a,s')^\t\w_{n+1}$, as long as we have $\w_{n+1}$  
we can immediately compute $Q^{\piexpi}_{n+1}(s,a) = \vpsi^{\piexpi}(s,a)^\t\w_{n+1}$.
This reduces the computation of all ${Q}^{\piexpi}_{n+1}$ to the problem of determining $\w_{n+1}$, which can be posed as a supervised learning problem whose objective is to minimise some loss derived from~(\ref{eq:reward}). Once ${Q}^{\piexpi}_{n+1}$ have been computed, we can apply GPI to derive a policy $\pi$ that is no worse, and possibly better, than $\piexp_1, ..., \piexp_n$ on task $M_{n+1}$.

\section{Extending the notion of environment}

\ct{barreto2017successor} focus on environments of the form~(\ref{eq:M}). In this paper we propose a more general notion of environment: 
\begin{equation}
\label{eq:M_ext}
\MM(\S, \A, p, \gamma) \equiv \{ M(\S, \A, p, \cdot, \gamma)\}.
\end{equation}
\MM\ contains \emph{all} MDPs that share the same \S, \A, $p$, and $\gamma$, regardless of whether their rewards can be computed as a linear combination of the features $\vphi$. Clearly, $\MM \supset \M$. 

We want to devise a transfer framework for environment \MM. Ideally this framework should have two properties. First, it should be principled, in the sense that we should be able to provide theoretical guarantees regarding the performance of the transferred policies. Second, it should give rise to simple methods that are applicable in practice, preferably in combination with deep learning. Surprisingly, we can have both of these properties by simply looking at \sfgpi\ from a slightly different point of view, as we show next.

\subsection{Guarantees on the extended environment}

\ct{barreto2017successor} provide theoretical guarantees on the performance of \sfgpi\ applied to any task $M \in \M$. In this section we show that it is in fact possible to derive guarantees for any task in \MM. Our main result is below: 

\begin{proposition}
\label{teo:beyond_linearity}
Let $M \in \MM$ and let $Q^{\pi_j^{*}}_i$ be the action-value function of an optimal 
policy of $M_j \in \MM$ when executed in $M_i \in \MM$. Given approximations 
$\{\Qt^{\pi_1}_i, \Qt^{\pi_2}_i, ..., \Qt^{\pi_n}_i\}$ such that 
$
 \left|Q^{\pi_j^{*}}_i(s, a) - \Qt^{\pi_j}_i(s,a) \right| \le \epsilon
$
for all $s \in \S$, $a \in \A$, and $j \in \{1, 2, ..., n\}$,
let 
\begin{equation}
\label{eq:gpi_policy}
\pihmax(s) \in \argmax_a {\max}_j \Qt^{\pi_j}_i (s,a).
\end{equation}
Then,
{\small
\begin{equation}
\label{eq:extrapolation_bound}
\|Q^{*} - Q^{\pi}\|_{\infty} 
\le \dfrac{2}{1-\gamma} \left( 
\| r - r_i \|_{\infty} +
\mathop{\min}_j \|r_i - r_j \|_{\infty} + \epsilon
\right),
\end{equation}
}
\hspace{-2mm} where $Q^*$ is the optimal value function of $M$, $Q^\pi$ is the value function of $\pi$ in $M$, and $\|f - g\|_\infty = \max_{s,a}|f(s,a) - g(s,a)|$.
\end{proposition}

The proof of Proposition~\ref{teo:beyond_linearity} is in the supplementary material. Our result provides guarantees on the performance of the GPI policy~(\ref{eq:gpi_policy}) applied to an MDP $M$ with \emph{arbitrary} reward function $r(s,a)$. 
Note that, although the proposition does not restrict any of the tasks to be in \M, in order to compute the GPI policy~(\ref{eq:gpi_policy}) we still need an efficient way of evaluating the policies $\pi_j$ on task $M_i$. As explained in Section~\ref{sec:background}, one way to accomplish this is to assume that $M_i$ and all $M_j$ appearing in the statement of the proposition belong to \M; this allows the use of SFs to quickly compute $\Qt^{\pi_j}_i (s,a)$.

Let us thus take a closer look at Proposition~\ref{teo:beyond_linearity} under the assumption that all MDPs belong to \M, except perhaps for $M$. The proposition is based on a reference MDP $M_i \in \M$. If $M_i = M_j$ for some $j$, the second term of~(\ref{eq:extrapolation_bound}) disappears, and we end up with a lower bound on the performance of a policy computed for $M_i$ when executed in $M$. More generally, one should think of $M_i$ as the MDP in \M\ that is ``closest'' to $M$ in some sense (so it may be that $M_i \ne M_j$ for all $j$). Specifically, if $r_i(s,a,s') = \vphi(s,a,s')^\t \w_i$, we can think of $\w_i$ as the vector that provides the best approximation $\vphi(s,a,s')^\t \w_i \approx r(s,a, s')$ under some well-defined criterion. The first term of~(\ref{eq:extrapolation_bound}) can thus be seen as the ``distance'' between $M$ and \M, which suggests that the performance of \sfgpi\ should degrade gracefully as we move away from the original environment \M. In the particular case where $M \in \M$ the first term of~(\ref{eq:extrapolation_bound}) vanishes, and we recover \ctp{barreto2017successor} Theorem~2. 

\subsection{Uncovering the structure of the environment}
\label{sec:uncovering}

We want to solve a subset of the tasks in \MM\ and use GPI to promote transfer between these tasks. In order to do so efficiently, we need a function $\vphi$ that covers the tasks of interest as much as possible. If we were to rely on \ctp{barreto2017successor} original results, we would have guarantees on the performance of the GPI policy only if \vphi\ spanned \emph{all} the tasks of interest. Proposition~\ref{teo:beyond_linearity} allows us to remove this requirement, since now we have theoretical guarantees for any choice of $\vphi$. In practice, though, we want features \vphi\ such that the first term of~(\ref{eq:extrapolation_bound}) is small for all tasks of interest.

There might be contexts in which we have direct access to features $\vphi(s,a,s')$ that satisfy~(\ref{eq:reward}), either exactly or approximately. Here though we are interested in the more general scenario where this structure is not available, nor given to the agent in any form. In this case the use of SFs requires a way of unveiling such a structure. 

\ct{barreto2017successor} assume the existence of a $\vphi \in \R^{d}$ that satisfy~(\ref{eq:reward}) exactly and formulate the problem of computing an approximate $\tvphi$ as a multi-task learning problem. The problem is decomposed into $D$ regressions, each one associated with a task. For reasons that will become clear shortly we will call these tasks \emph{base tasks} and denote them by $\MMtr \equiv \{\Mtr_1, \Mtr_2,..., \Mtr_D\} \subset \M$. The multi-task problem thus consists in solving the approximations 
\begin{equation}
\label{eq:multitask}
\tvphi(s,a,s')^\t \wt_i \approx \rtr_i(s,a,s'), \text{ for } i = 1, 2, ..., D, 
\end{equation}
where $\rtr_i$ is the reward of $\Mtr_i$~\cp{caruana97multitask,baxter2000model}.

In this section we argue that~(\ref{eq:multitask}) can be replaced by a much simpler approximation problem. Suppose for a moment that we know a function $\vphi$ and $D$ vectors $\w_i$ that satisfy~(\ref{eq:multitask}) exactly. If we then stack the vectors $\w_i$ to obtain a matrix $\WW \in \R^{D \times d}$, we can write $\r(s,a,s') = \WW \vphi(s,a,s')$, where the \ith\ element of $\rvtr(s,a,s') \in \R^{D}$ is $r_i(s,a,s')$. Now, as long as we have $d$ linearly independent tasks $\w_i$, we can write $\vphi(s,a,s') = (\WW^\t\WW)^{-1}\WW^\t \r(s,a,s') = \WW^{\dag} \r(s,a,s')$. Since $\vphi$ is given by a linear transformation of \r, any task representable by the former can also be represented by the latter. To see why this is so, note that for any task in \M\ we have $r(s,a,s') = \w^\t\vphi(s,a,s') = \w^\t \WW^{\dag} \r(s,a,s')= (\w')^\t\r(s,a,s')$. Therefore, we can use the rewards \r\ themselves as features, which means that we can replace~(\ref{eq:multitask}) with the much simpler approximation
\begin{equation}
\label{eq:simpler_approx}
\tvphi(s,a,s') = \rvt(s,a,s') \approx \r(s,a,s').
\end{equation}

One potential drawback of directly approximating  $\rvtr(s,a,s')$ is that we no longer have the flexibility of distinguishing between $d$, the dimension of the environment \M, and $D$, the number of base tasks in \MMtr. Since in general we will solve~(\ref{eq:multitask}) or~(\ref{eq:simpler_approx}) based on data, having $D > d$ may lead to a better approximation. On the other hand, using $\rvt(s,a,s')$ as features has a number of advantages that in many scenarios of interest should largely outweigh the loss in flexibility. 

One such advantage becomes clear when we look at the scenario above from a different perspective. Instead of assuming we know a \vphi\ and a \WW\ that satisfy~(\ref{eq:multitask}), we can note that, given any set of $D$ tasks $r_i(s,a,s')$, $k \le D$ of them will be linearly independent. Thus, the reasoning above applies without modification. Specifically, if we use the tasks's rewards as features, as in~(\ref{eq:simpler_approx}), we can think of them as \emph{inducing} a $\vphi(s,a,s') \in \R^k$---and, as a consequence, an environment \M. This highlights a benefit of replacing~(\ref{eq:multitask}) with~(\ref{eq:simpler_approx}): the fact that we can work directly in the space of tasks, which in many cases may be more intuitive. When we think of \vphi\ as a non-observable, abstract, quantity, it can be difficult to define what exactly the base tasks should be. In contrast, when we look at \rvtr\ directly the question we are trying to answer is: is it possible to approximate the reward function of all tasks of interest as a linear combination of the functions $\rtr_i$? As one can see, the set \MMtr\ works exactly as a basis, and thus the name ``base tasks''.

Another interesting consequence of using an approximation $\tvphi(s,a,s') \approx \rvtr(s,a,s')$ as features is that the resulting SFs are nothing but ordinary action-value functions. Specifically, the \jth\ component of $\tvpsi^{\pi_i}(s,a)$ is simply $\Qt^{\pi_i}_j(s,a)$. Next we discuss how this gives rise to a simple approach that can be combined with deep learning in a stable way.

\section{Transfer in deep reinforcement learning}
\label{sec:deep_transfer}

As described in the introduction, we are interested in using \sfgpi\ to build scalable agents that can be combined with deep learning in a stable way. Since deep learning generally involves vast amounts of data whose storage is impractical, we will be mostly focusing on methods that work online. As a motivating example for this section we show in Algorithm~\ref{alg:sfgpi} how SFs and GPI can be combined with \ctp{watkins92qlearning} $Q$-learning.\footnote{We use $x \xleftarrow{\alpha} y$ meaning $x \leftarrow x + \alpha y$. We also use $\nabla_{\params} f(x)$ to denote the gradient of $f(x)$ with respect to the parameters $\params$.} 

\setlength{\textfloatsep}{5pt}
\begin{algorithm}
   \caption{\sfgpi\ with $\epsilon$-greedy $Q$-learning} 
   \label{alg:sfgpi}
\begin{algorithmic}[1]
\REQUIRE
$\left\{\begin{array}{ll}
\tvphi, \; \tPsi \equiv \{\tvpsi^{\pi_1}, ..., \tvpsi^{\pi_n}\} & \text{features, SFs} \\
\eb\ & \text{learn a new SF?} \\
\alpha_{\psi}, \alpha_w, \epsilon, \ns  & \text{hyper-parameters} \\
\end{array}
\right.$
\IF{\eb}
\STATE create $\tvpsi^{\piexp_{n+1}}$ parametrised by $\params_\psi$
\STATE $\tPsi \la \tPsi \cup \{\tvpsi^{\piexp_{n+1}}\}$ \label{it:add_psi}
\ENDIF
\STATE select initial state $s \in \S$
\FOR{$\ns$ steps}
\STATE {\bf if} Bernoulli($\epsilon$)=1 {\bf then} 
$a \la$ Uniform($\A$) 
\COMMENT{exploration}
\STATE {\bf else} $a \la \argmax_b \max_i \tvpsi^{\piexp_i}(s,b)^\t \wt$
\COMMENT{GPI  \label{it:gpi} }
\STATE Execute action $a$ and observe $r$ and $s'$
\STATE $\wt \xleftarrow{\alpha_w} \left[r - \tvphi(s,a, s')^\t \wt \right]\tvphi(s,a, s')$
\COMMENT{learn \wt\ \label{it:learn_w} }
\IF[will learn new SFs]{\eb}
\STATE $a' \la \argmax_b \tvpsi^{\piexp_{n+1}}(s,b)^\t \wt$
\FOR{$i \la 1, 2, ..., d$}
\STATE  {\small $\delta \la \tvphi_i(s,a,s') + \gamma \tvpsi_i^{\piexp_{n+1}}(s', a') - 
\tvpsi^{\piexp_{n+1}}_i(s,a)$} \label{it:td}
\STATE  {\small $\params_\psi \xleftarrow{\alpha_{\psi}} \delta \nabla_{\params_\psi} \tvpsi^{\piexp_{n+1}}_i(s,a)$} 
\COMMENT{learn $\tvpsi^{\piexp_{n+1}}$ \label{it:learn_psi}}
\ENDFOR
\ENDIF
\STATE {\bf if} $s'$ is not terminal {\bf then} $s \la s'$ 
\STATE {\bf else} select initial state $s \in \S$
\ENDFOR
\end{algorithmic}
\end{algorithm}

Algorithm~\ref{alg:sfgpi} differs from standard $Q$-learning in two main ways. First, instead of selecting actions based on the value function being learned, the behaviour of the agent is determined by GPI (line~\ref{it:gpi}). Depending on the set of SFs \tPsi\ used by the algorithm, this can be a significant improvement over the greedy policy induced by $\Qt^{\pi_{n+1}}$, which usually is the main determinant of a $Q$-learning agent's behaviour. 

Algorithm~\ref{alg:sfgpi} also deviates from conventional $Q$-learning in the way a policy is learned. There are two possibilities here. One of them is for the agent to rely exclusively on the GPI policy computed over \tPsi\ (when the variable \eb\ is set to false). In this case no specialised policy is learned for the current task, which reduces the RL problem to the supervised problem of determining \wt\ (solved in line~\ref{it:learn_w} as a least-squares regression). 

Another possibility is to use data collected by the GPI policy to learn a policy $\pi_{n+1}$ specifically tailored for the task. As shown in lines~\ref{it:td} and~\ref{it:learn_psi}, this comes down to solving equation~(\ref{eq:bellman_psi}). When $\pi_{n+1}$ is learned the function 
$\Qt^{\pi_{n+1}}(s,a) = \tvpsi^{\pi_{n+1}}(s,a)^\t \wt$ 
also takes part in GPI. This means that, if the approximations $\Qt^{\pi_{i}}(s,a)$ are reasonably accurate, the policy computed by Algorithm~\ref{alg:sfgpi} should be strictly better than $Q$-learning's counterpart. The SFs  $\tvpsi^{\pi_{n+1}}$ can then be added to the set \tPsi, and therefore a subsequent execution of Algorithm~\ref{alg:sfgpi} will result in an even stronger agent. The ability to build and continually refine a set of skills is widely regarded as a desirable feature for continual (or lifelong) learning~\cp{thrun96learning}. 

\subsection{Challenges involved in building features}
\label{sec:difficulty}

In order to use Algorithm~\ref{alg:sfgpi}, or any variant of the \sfgpi\ framework, we need the features $\tvphi(s,a,s')$. A natural way of addressing the computation of $\tvphi(s,a,s')$ is to see it as a multi-task problem, as in~(\ref{eq:multitask}). We now discuss the difficulties involved in solving this problem online at scale. 

The solution of~(\ref{eq:multitask}) requires data coming from $D$ base tasks. In principle, we could look at the collection of sample transitions as a completely separate process. However, here we are assuming that the base tasks \MMtr\ are part of the RL problem, that is, we want to collect data using policies that are competent in \MMtr.
In order to maximise the potential for transfer, while learning the policies $\pi_i$ for the base tasks $\Mtr_i$ we should also learn the associated SFs $\tvpsi^{\pi_i}$; this corresponds to building the initial set \tPsi\ used by Algorithm~\ref{alg:sfgpi}. Unfortunately, learning value functions in the form~(\ref{eq:sf}) while solving~(\ref{eq:multitask}) can be problematic. Since the SFs $\tvpsi^{\pi_i}$ depend on $\tvphi$, learning the former while refining the latter can clearly lead to undesirable solutions (this is akin to having a non-stationary reward function). On top of that, the computation of \tvphi\ itself depends on the SFs $\tvpsi^{\pi_i}$, for they ultimately define the data distribution used to solve~(\ref{eq:multitask}). This circular dependency can make the process of concurrently learning $\tvphi$ and $\tvpsi^{\pi_i}$ unstable---something we observed in practice. 

One possible solution is to use a conventional value function representation while solving~(\ref{eq:multitask}) and only learn SFs for the subsequent tasks~\cp{barreto2017successor}. This has the disadvantage of not reusing the policies learned in \MMtr\ for transfer. Alternatively, one can store all the data and learn the SFs associated with the base tasks only \emph{after} \tvphi\ has been learned, but this may be difficult in scenarios involving large amounts of data. Besides, any approximation errors in $\tvphi$ will already be reflected in the initial \tPsi\ computed in \MMtr.

\subsection{Learning features online while retaining transferable knowledge}
\label{sec:learn_sfs}

To recapitulate, we are interested in solving $D$ base tasks $\Mtr_i$ and, while doing so, build \tvphi\ and the initial set \tPsi\ to be used by Algorithm~\ref{alg:sfgpi}. We want \tvphi\ and \tPsi\ to be learned concurrently, so we do not have to store transitions, and preferably \tPsi\ should not reflect approximation errors in~\tvphi. 

We argue that one can accomplish all of the above by replacing~(\ref{eq:multitask}) with~(\ref{eq:simpler_approx}), that is, by directly approximating $\rvtr(s,a,s')$, which is observable, and adopting the resulting approximation as the features \tvphi\ required by Algorithm~\ref{alg:sfgpi}. 

As discussed in Section~\ref{sec:uncovering}, when using rewards as features the resulting SFs are collections of value functions: $\tvpsi^{\pi_i} = \vQt^{\pi_i} \equiv [\Qt^{\pi_i}_1, \Qt^{\pi_i}_2, ..., \Qt^{\pi_i}_D]$. This leads to a particularly simple way of building the features \tvphi\ while retaining transferable knowledge in $\tPsi$. Given a set of $D$ base tasks $\Mtr_i$, while solving them we only need to carry out two extra operations: compute approximations $\tilde{r}_i(s,a,s')$ of the functions $\rtr_i(s,a,s')$, to be used as \tvphi, and evaluate the resulting policies on all tasks---{\sl i.e.,} compute $\Qt^{\pi_i}_j$---to build \tPsi.

Before providing a practical method to compute \tvphi\ and $\tPsi$, we note that, although the approximations $\tilde{r}_i(s,a,s')$ can be learned independently from each other, the computation of $\vQt^{\pi_i}$ requires policy $\pi_i$ to be evaluated under different reward signals. This can be accomplished in different ways; here we assume that the agent is able to interact with the tasks $\Mtr_i$ in parallel. We can consider that at each transition the agent observes rewards from all the base tasks, $\rvtr \in \R^D$, or a single scalar $\rtr_i$ associated with one of them. We will assume the latter, but our solution readily extends to the scenario where the agent simultaneously observes $D$ tasks.

Algorithm~\ref{alg:sfgpi_basis} shows a possible way of implementing our solution, again using $Q$-learning as the basic RL method. We highlight the fact that GPI can already be used in this phase, as shown in line~\ref{it:gpi2}, which means that the policies $\pi_i$ can ``cooperate'' to solve each task $\Mtr_i$.
\begin{algorithm}
   \caption{Build \sfgpi\ basis with $\epsilon$-greedy $Q$-learning} 
   \label{alg:sfgpi_basis}
\begin{algorithmic}[1]
\REQUIRE 
$\left\{\begin{array}{ll}
\Mtr_1, \Mtr_2, ..., \Mtr_D & \text{base tasks} \\
\alpha_Q, \alpha_r, \epsilon, \ns\  & \text{hyper-parameters} \\
\end{array}
\right.$
\FOR{$\ns$ steps}
\STATE select a task $t \in \{1, 2, ..., D\}$ and a state $s \in \S$
\label{it:task_selection}
\STATE {\bf if} Bernoulli($\epsilon$)=1 {\bf then} 
$a \la$ Uniform($\A$) 
\COMMENT{exploration}
\STATE {\bf else} $a \la \argmax_b \max_i \Qt^{\piexp_i}_t(s,b)$
\COMMENT{GPI \label{it:gpi2}}
\STATE Execute action $a$ in $\Mtr_t$ and observe $\rtr$ and $s'$
\STATE $\params_r \xleftarrow{\alpha_r} \left[\rtr - \rt_t(s,a,s')\right]
\nabla_{\params_r} \rt_t(s,a,s')$ 
\label{it:learn_r}
\FOR{$i \la 1, 2, ..., D$}
\STATE $a' \la \argmax_b \Qt^{\piexp_i}_i(s,b)$
\COMMENT{$a' \equiv \piexp_i(s)$}
\STATE {\small $\params_Q \xleftarrow{\alpha_Q} 
\left[r + \gamma \Qt^{\piexp_i}_t(s', a') - 
\Qt^{\piexp_i}_t(s,a)\right] \nabla_{\params_Q} \Qt^{\piexp_i}_t(s,a)$} 
\label{it:learn_q}
\ENDFOR
\ENDFOR
\STATE {\bf return} $\tvphi \equiv \left[ \tilde{r}_1, ..., \tilde{r}_D \right]$ and $\tPsi \equiv \{ \vQt^{\pi_1}, ..., \vQt^{\pi_D}\}$
\end{algorithmic}
\end{algorithm}

Note that if we were to learn general $\tvpsi^{\pi_i}(s,a)$ and $\tvphi(s,a,s')$ in parallel we would necessarily have to use the latter to update the former, which means computing an approximation on top of another (see~(\ref{eq:bellman_psi})). In contrast, when learning $\Qt^{\pi_i}_j(s,a)$ we can use the \emph{actual} rewards $\rtr_j(s,a,s')$, as opposed to the approximations $\rt_j(s,a,s')$ (line~\ref{it:learn_q} of Algorithm~\ref{alg:sfgpi_basis}). This means that $\tvphi$ and $\tPsi$ can be learned concurrently without the accumulation of approximation errors in $\tvpsi^{\pi_i}$, as promised.

\section{Experiments}
\label{sec:experiments}

In this section we use experiments to assess whether the proposed approach can indeed promote transfer on large-scale domains. Here we focus on what we consider the most relevant aspects of our experiments; for further details and additional results please see the supplementary material.

 \subsection{Environment}
 \label{sec:environment}
 
 The environment we consider is conceptually similar to one of the problems used by~\ct{barreto2017successor} to evaluate their framework: the agent has to navigate in a room picking up desirable objects while avoiding undesirable ones. Here the problem is tackled in a particularly challenging format: instead of observing the state $s_t$ at time step $t$, as in the original experiments, the agent interacts with the environment from a first-person perspective, only receiving as observation a $84 \times 84$ image $o_t$ that is insufficient to disambiguate the actual underlying state of the MDP (see Figure~\ref{fig:observation}). 
 
\begin{figure}
\centering
\subfigure[Screenshot of environment\label{fig:observation}]{
\centering
\begin{minipage}{0.25\textwidth}
\centering
\vspace{1mm}
\includegraphics[scale=0.09]{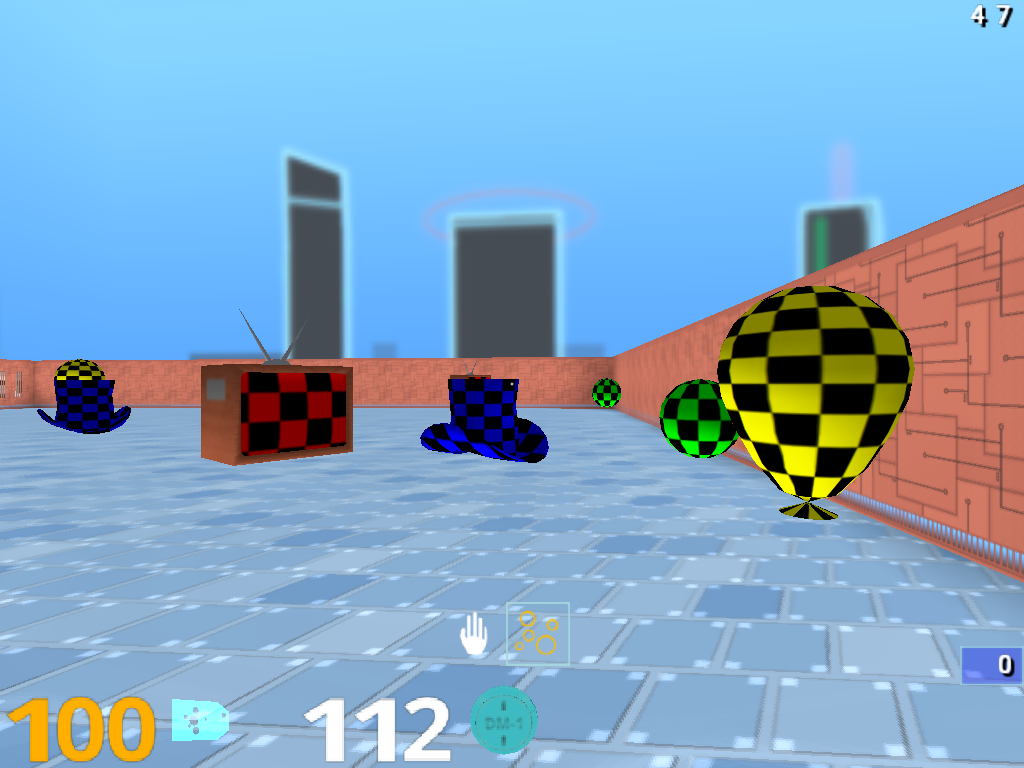}
\vspace{1mm}
\end{minipage}
}
\subfigure[Base tasks \MMtr\ \label{fig:base_tasks}]{
\centering
\begin{minipage}{0.2\textwidth}
\setlength\tabcolsep{1pt} 
\begin{tabular}{|c|cccc|}
\multicolumn{1}{c}{} 
 & \includegraphics[scale=0.04]{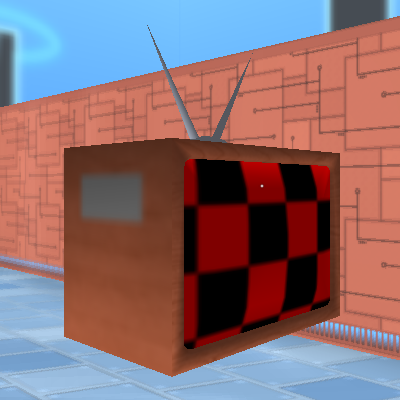} 
 & \includegraphics[scale=0.04]{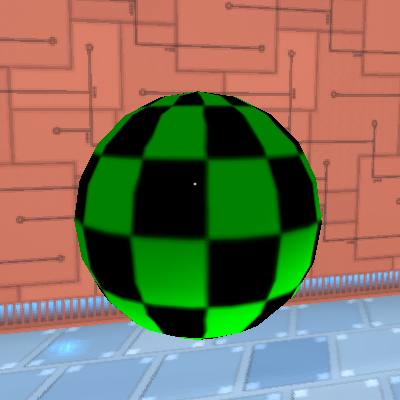} 
 & \includegraphics[scale=0.04]{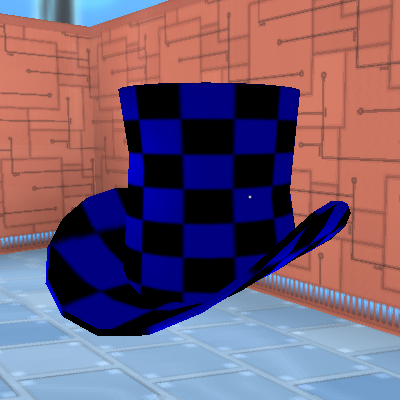} 
 & \multicolumn{1}{c}{\includegraphics[scale=0.04]{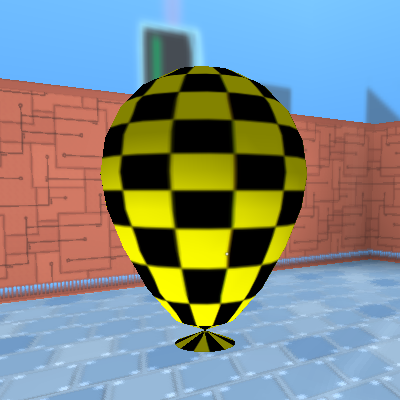}} \\ \hline
 $\Mtr_1$ &  1 & 0 & 0 & 0 \\
 $\Mtr_2$  & 0 & 1 & 0 & 0 \\
 $\Mtr_3$  & 0 & 0 & 1 & 0 \\
 $\Mtr_4$  & 0 & 0 & 0 & 1 \\ \cline{1-5}
   \end{tabular}
\end{minipage}}
\caption{Environment and base tasks.}
\end{figure}

We used \ctp{beattie2016deepmind} DeepMind Lab platform to design our 3D environment, which works as follows. The agent finds itself in a room full of objects of different types. There are five instances of each object type: ``TV'', ``ball'', ``hat'', and ``balloon''. Whenever the agent hits an object it picks it up and another object of the same type appears at a random location in the room. This process goes on for a minute, after which a new episode starts.

The type of an object determines the reward associated with it; thus, a task is defined (and can be referred to) by four numbers indicating the rewards attached to each object type. 
For example, in task \taskexample\ the agent is interested in objects of the first type and should avoid objects of the second type, while the other object types are irrelevant.
We defined base tasks that will be used by Algorithm~\ref{alg:sfgpi_basis} to build $\tvphi$ and \tPsi\ (see Figure~\ref{fig:base_tasks}). The transfer ability of the algorithms will be assessed  on different, unseen, tasks, referred to as \emph{test tasks}.

 \subsection{Agents}

 The \sfgpi\ agent adopted in the experiments is a variation of Algorithms~\ref{alg:sfgpi} and~\ref{alg:sfgpi_basis} that uses \ctp{watkins89learning} $Q(\lambda)$ to apply $Q$-learning with eligibility traces. The functions $\tvphi$ and $\tPsi$ are computed by a deep neural network whose architecture is shown in Figure~\ref{fig:architecture}. The network is composed of three parts. The first one uses the history of observations and actions up to time $t$, $h_t$, to compute a state signal $\tilde{s}_t = f(h_t)$. The construction of $\tilde{s}_t$ can itself be broken into two stages corresponding to specific functional modules: a convolutional network (CNN) to handle the pixel-based observation $o_t$ and a long short-term network (LSTM) to compute $f(h_t)$ in a recursive way~\cp{lecun1998gradient,hochreiter1997long}. The second part of the network is composed of $D+1$ specialised blocks that receive $\tilde{s}_t$ as input and compute $\tvphi(\tilde{s}_t, a)$ and $\tvpsi^{\pi_i}(\tilde{s}_t, a)$ for all $a \in \A$. Each one of these blocks is a multilayer perceptron (MLP) with a single hidden layer~\cite{rumelhart86learning}. 
The third part of the network is simply \wt, which combined with $\tvphi$ and $\tvpsi^{\pi_i}$ will provide the final approximations.

\begin{figure}
\begin{center}
\includegraphics[scale=0.28]{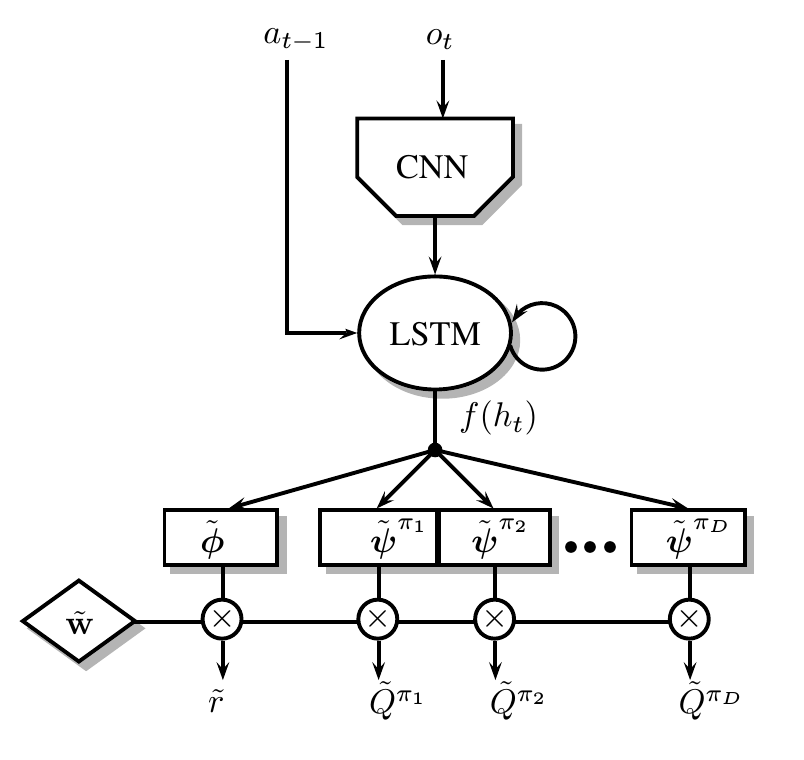}
\end{center}
\caption{Deep architecture used. Rectangles represent MLPs. \label{fig:architecture} }
\end{figure}

\begin{figure*}
\centering
\newcommand{\scl}{0.49}
\newcommand{\wdt}{0}
\subfigure[Base tasks \label{fig:res_base_tasks}]{
\includegraphics[scale=\scl]{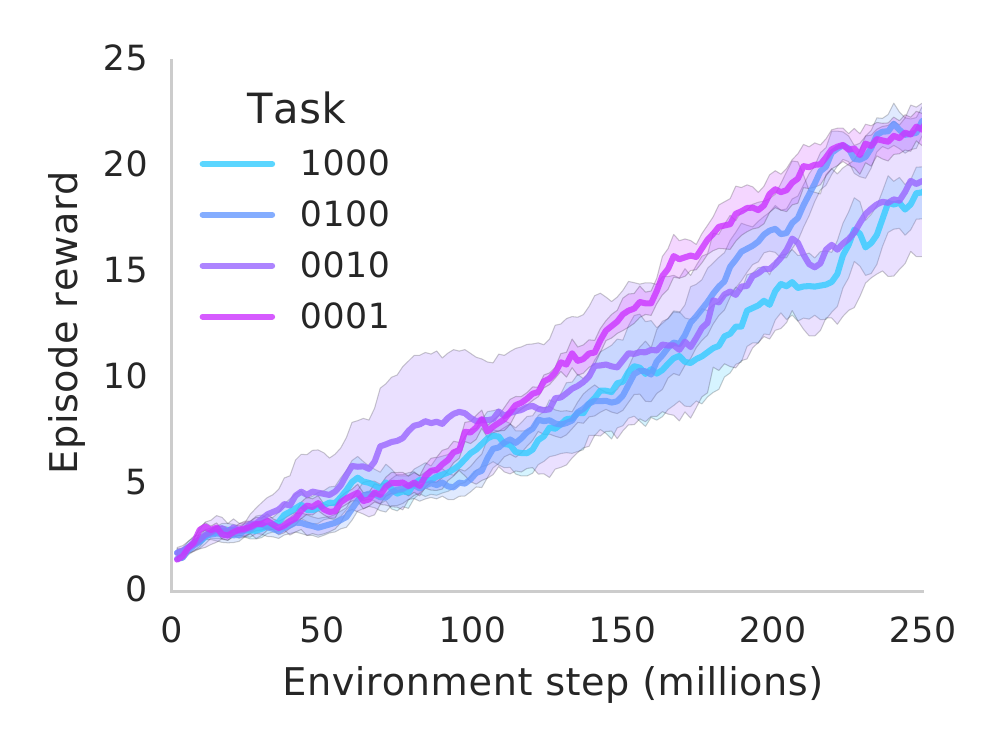}
 }
 \subfigure[Task \taskfive\ \label{fig:test_5}]{
 \includegraphics[scale=\scl]{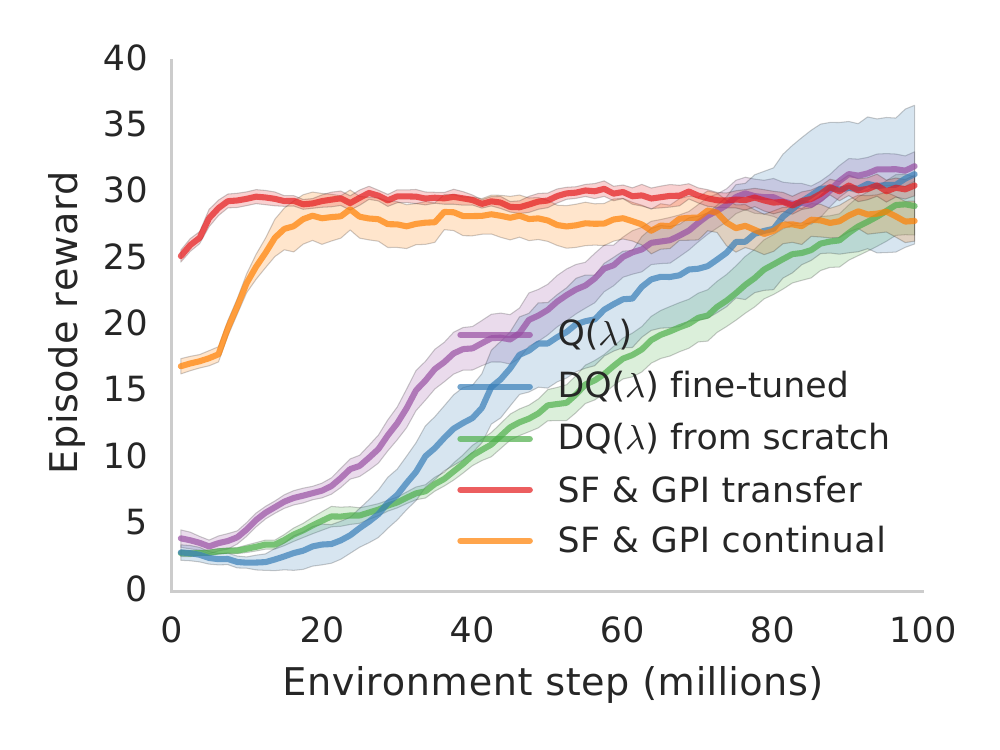}
 }
\subfigure[Task \tasknine\ \label{fig:test_9}]{
\includegraphics[scale=\scl]{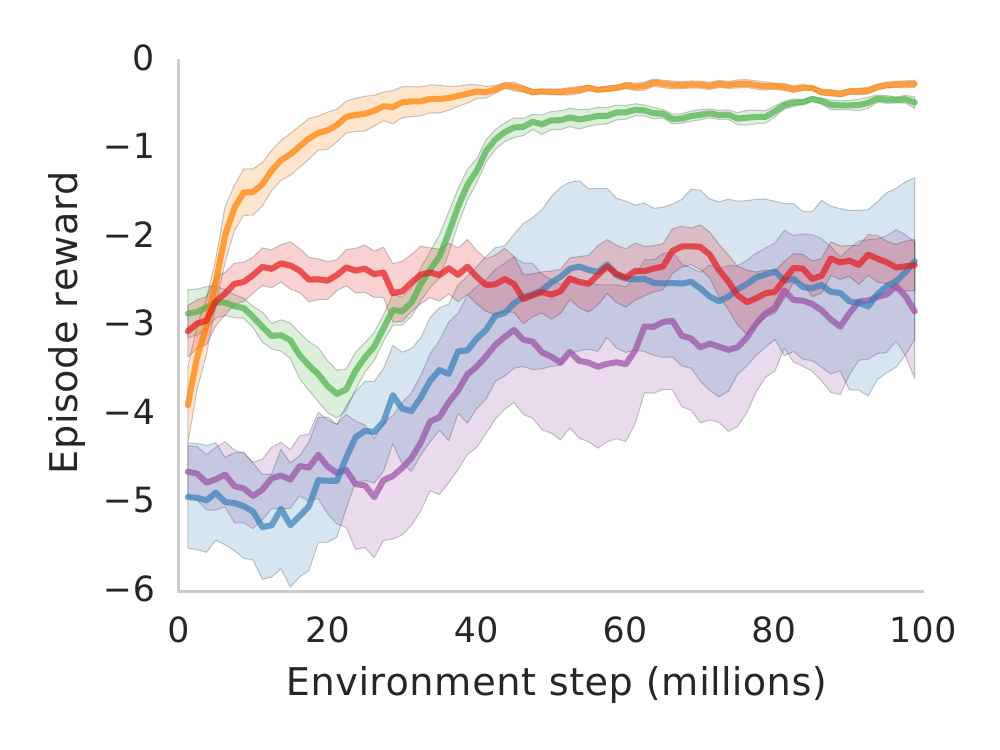}
}
 
 \subfigure[Task \taskeleven\ \label{fig:test_10}]{
 \includegraphics[scale=\scl]{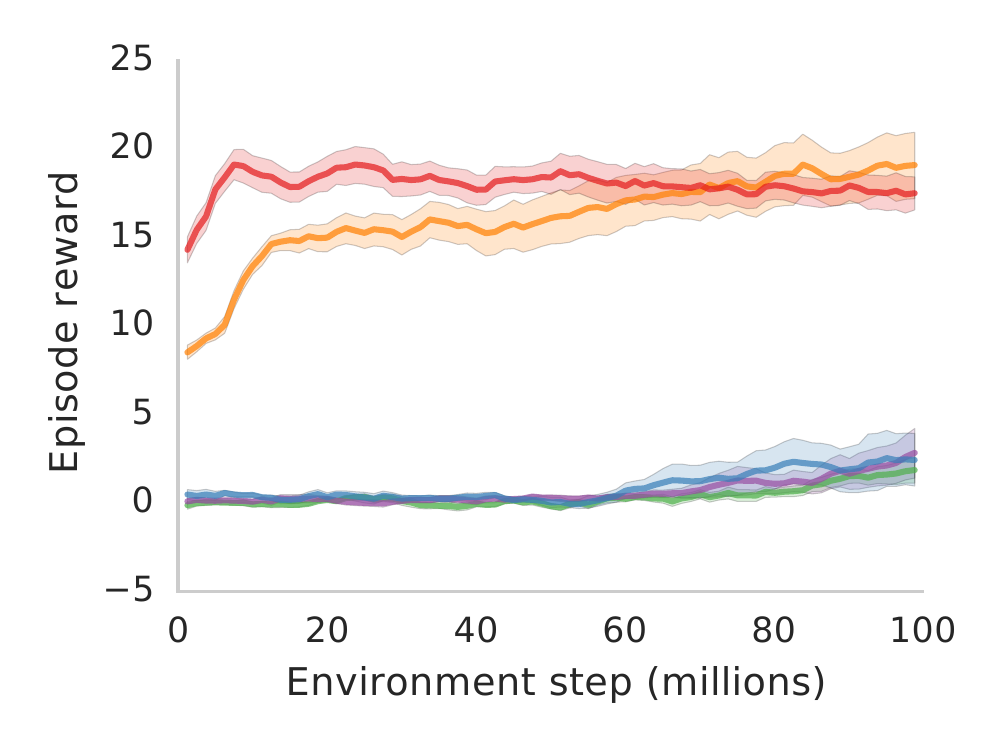}
 }
\subfigure[Task \tasktwelve\ \label{fig:test_12}]{
\includegraphics[scale=\scl]{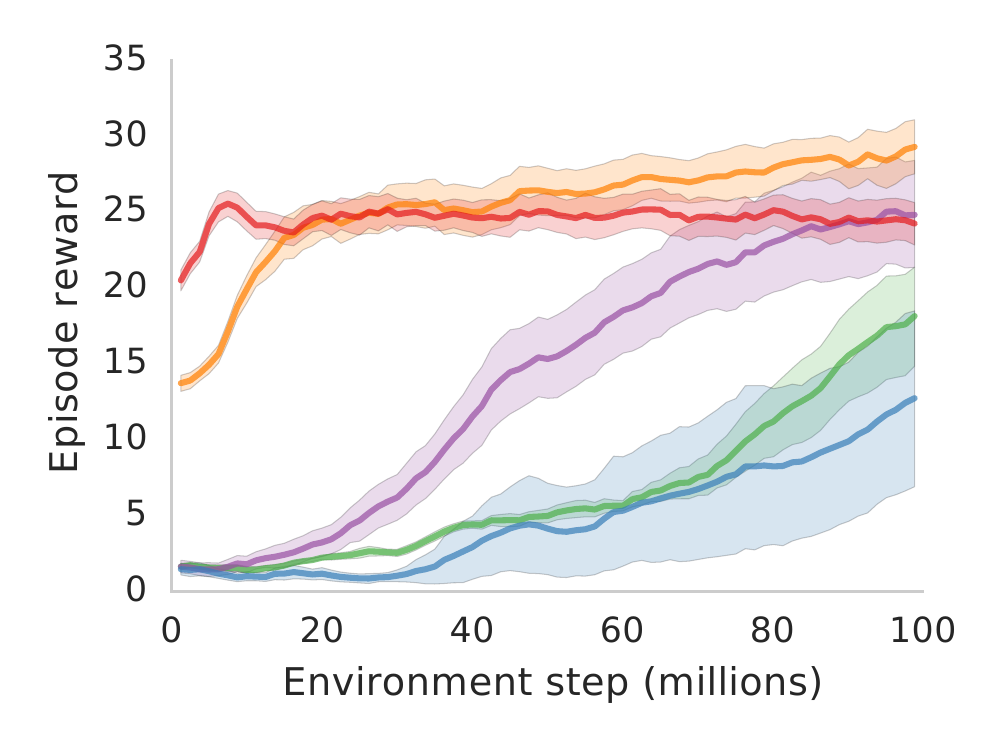}
}
\subfigure[Task \taskthirteen\ \label{fig:test_13}]{
\includegraphics[scale=\scl]{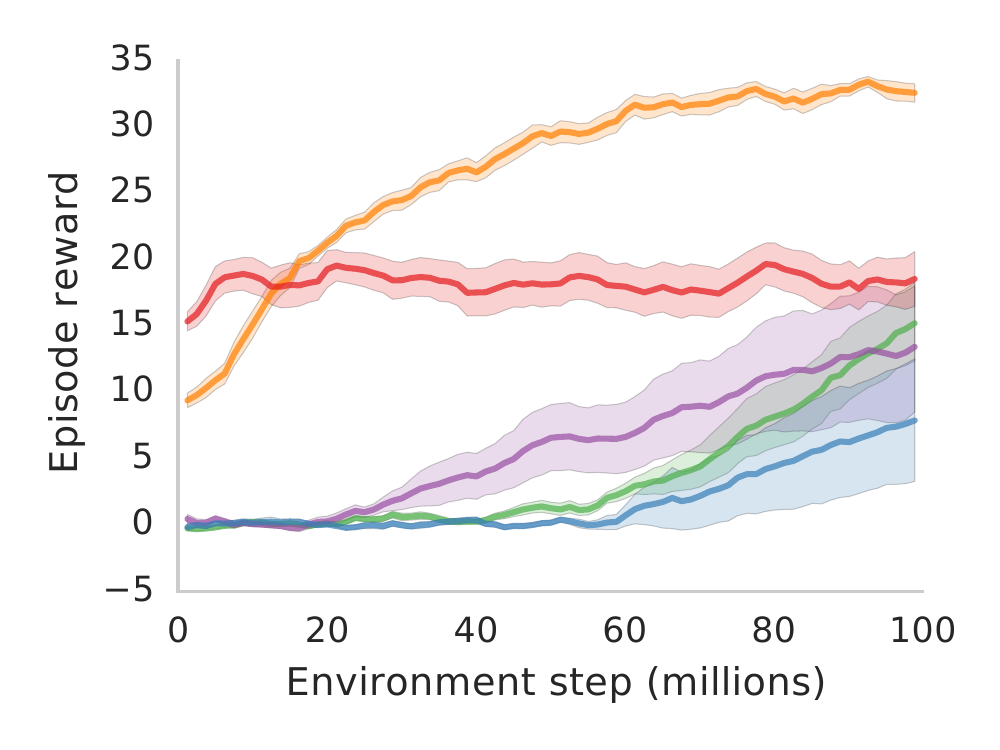}
}
\caption{Average reward per episode on the base tasks and a selection of test tasks. The $x$ axes have different scales because the amount of reward available changes across tasks. Shaded regions are one standard deviation over $10$ runs.
\label{fig:results}}
\vspace{-3mm}
\end{figure*}

The entire architecture was trained end-to-end through Algorithm~\ref{alg:sfgpi_basis} using the base tasks shown in Figure~\ref{fig:base_tasks}. After the \sfgpi\ agent had been trained it was tested on a test task, now using Algorithm~\ref{alg:sfgpi} with the newly-learned \tvphi\ and \tPsi. 
In order to handle the large number of sample trajectories needed in our environment both Algorithms~\ref{alg:sfgpi} and~\ref{alg:sfgpi_basis} used the IMPALA distributed   architecture~\cp{espehold2018impala}. 

Algorithm~\ref{alg:sfgpi} was run with and without the learning of a specialised policy (controlled via the variable \eb). We call the corresponding versions of the algorithm ``\sfgpirl'' and ``\sfgpisl'', respectively. 
We compare \sfgpi\ with baseline agents that use the same network architecture, learning algorithm, and distributed data processing. The only difference is in the way the network shown in Figure~\ref{fig:architecture} is updated and used during the test phase. Specifically, we ignore the MLPs used to compute $\tvphi$ and $\tvpsi^{\pi_i}$ and instead add another MLP, with the exact same architecture, to be jointly trained with \wt\ through $Q(\lambda)$. We then distinguish three baselines. The first one uses the state signal $\tilde{s}_t = f(h_t)$ learned in the base tasks to compute an approximation $\Qt(\tilde{s}, a)$---that is, both the CNN and the LSTM are fixed. We will refer to this method simply as $Q(\lambda)$. The second baseline is allowed to modify $f(h_t)$ during test, so we call it ``$DQ(\lambda)$ fine tuning'' as a reference to its deep architecture. Finally, the third baseline, ``$DQ(\lambda)$ from scratch'', learns its own representation $f(h_t)$.

\subsection{Results and discussion}
\label{sec:results}

Figure~\ref{fig:results} shows the results of \sfgpi\ and the baselines on a selection of test tasks. The first thing that stands out in the figures is the fact that \sfgpisl\ learns very good policies for the test tasks almost instantaneously. In fact, as this version of the algorithm is solving a simple supervised learning problem, its learning progress is almost imperceptible at the scale the RL problem unfolds. Since the baselines \emph{are} solving the full RL problem, in some tasks their performance eventually reaches, or even surpasses, that of the transferred policies. \sfgpirl\ combines the desirable properties of both \sfgpisl\ and the baselines. On one hand, it still benefits from the instantaneous transfer promoted by \sfgpi. On the other hand, its performance keeps improving, since in this case the transferred policy is used to learn a policy specialised to the current task. As a result, \sfgpirl\ outperforms the other methods in almost all of the tasks.\footnote{A video of \sfgpisl\ is included as a supplement, and can also be found on this link: \href{https://youtu.be/-dTnqfwTRMI}{https://youtu.be/-dTnqfwTRMI}.}

Another interesting trend shown in Figures~\ref{fig:results} is the fact that \sfgpi\ performs well on test tasks with negative rewards---in some cases considerably better than the alternative methods---, even though the agent only experienced positive rewards in the base tasks. This is an indication that the transferred GPI policy is combining the policies $\pi_i$ for the base tasks in a non-trivial way (line~\ref{it:gpi} of Algorithm~\ref{alg:sfgpi}).

In this paper we argue that \sfgpi\ can be applied even if assumption~(\ref{eq:reward}) is not strictly satisfied. In order to illustrate this point, we reran the experiments with \sfgpisl\ using a set of linearly-depend base tasks, $\MMtr' \equiv \{ {\tt 1000}, {\tt 0100}, {\tt 0011}, {\tt 1100} \}$. Clearly, $\MMtr'$ can only represent tasks in which the rewards associated with the third and fourth object types are the same. We thus fixed the rewards associated with the first two object types and compared the results of \sfgpisl\ using $\MMtr$ and $\MMtr'$ on several tasks where this is not the case. 
The comparison is shown in Figure~\ref{fig:basis_comparison}. As shown in the figure, although using a linearly-dependent set of base tasks does hinder transfer in some cases, in general it does not have a strong impact on the results. This smooth degradation of the performance is in accordance with Proposition~\ref{teo:beyond_linearity}. 

The result above also illustrates an interesting distinction between the space of reward functions and the associated space of policies. Although we want to be able to represent the reward functions of all tasks of interest, this does not guarantee that the resulting GPI policy will perform well. To see this, suppose we replace the positive rewards in \MMtr\ with negative ones. Clearly, in this case we would still have a basis spanning the same space of rewards; however, since now a policy that stands still is optimal in all tasks $M_i$, we should not expect GPI to give rise to good policies in tasks with positive rewards. One can ask how to define a ``behavioural basis'' that leads to good policies across \MM\ through GPI. We leave this as an interesting open question.

\begin{figure}
\begin{center}
\includegraphics[scale=0.38]{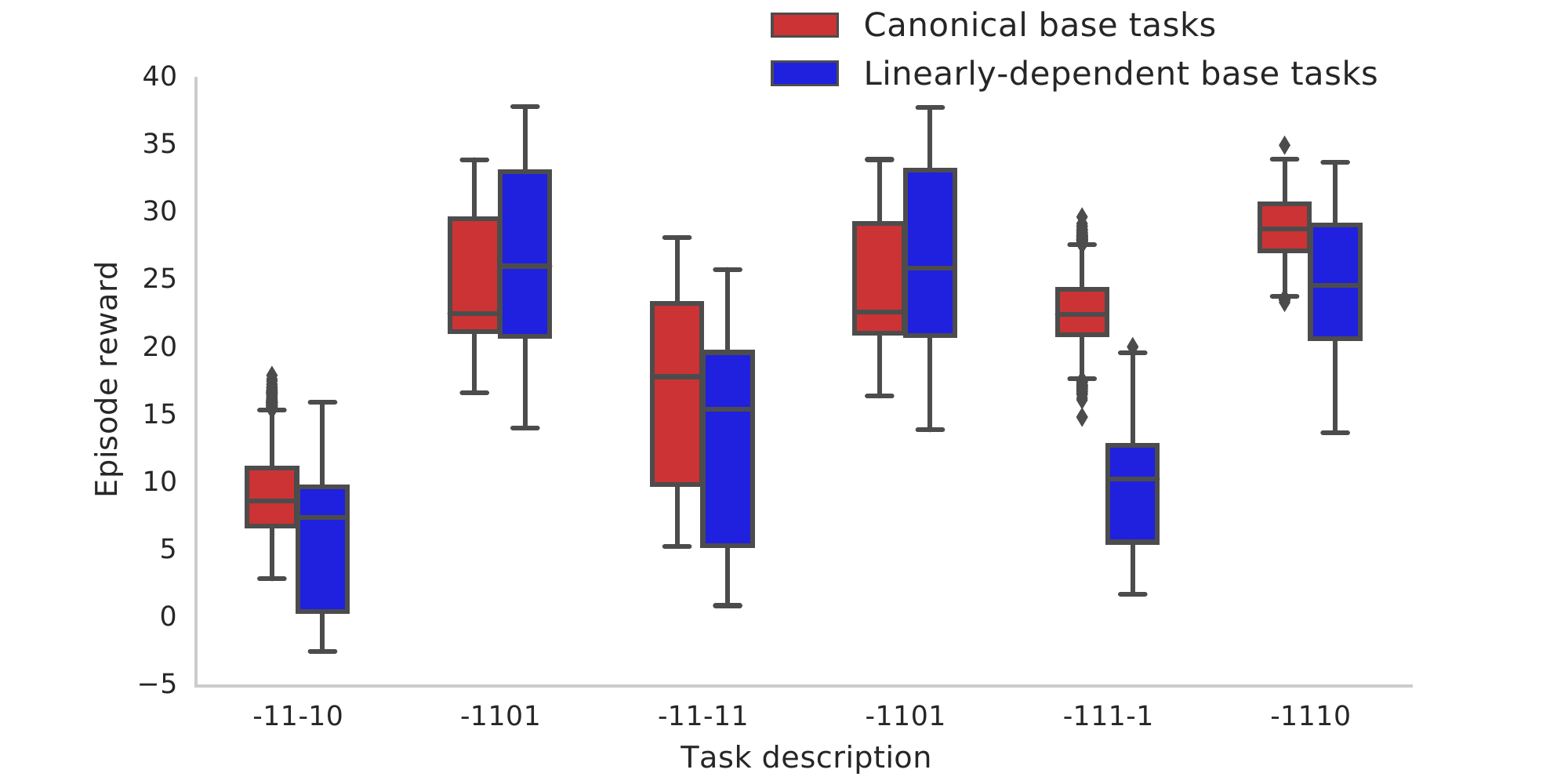}
\end{center}
\vspace{-2mm} 
\caption{Performance of \sfgpisl\ using base tasks $\MMtr$ and $\MMtr'$. The box plots summarise the distribution of the rewards received per episode between $50$ and $100$ million steps of learning. 
\label{fig:basis_comparison} }
\vspace{-2mm}
\end{figure}

\section{Related work}
\label{sec:related}

Recently, there has been a resurgence of the subject of transfer in the deep RL literature.
\ct{whye2017distral} propose an approach for the multi-task problem---which in our case is the learning of base tasks---that uses a shared policy as a regulariser for specialised policies. \ct{finn2017model} propose to achieve fast transfer, which is akin to \sfgpisl, by focusing on adaptability, rather than performance, during learning. \ct{rusu2016progressive} and \ct{kirkpatrick2016overcoming} introduce neural-network architectures well-suited for continual learning. There have also been previous attempts to combine SFs and deep learning for transfer, but none of them used GPI \cp{kulkarni2016deep,zhang2016deep}. 

Many works on the combination of deep RL and transfer propose modular network architectures that naturally induce a decomposition of the problem~\cp{devin2016learning,hees2017emergence,clavera2017policy}. Among these, a recurrent theme is the existence of sub-networks specialised in different skills that are managed by another network~\cp{hees2016learning,frans2017meta,oh2017zero}. This highlights an interesting connection between transfer learning and hierarchical RL, which has also recently re-emerged in the deep RL literature~\cp{vezhnevets2017feudal,bacon2017option}.

\vspace{-1mm}
\section{Conclusion}
\label{sec:conclusion}
\vspace{-1mm}

In this paper we extended the \sfgpi\ transfer framework in two ways. First, we showed that the theoretical guarantees supporting the framework can be extended to any set of MDPs that only differ in the reward function, regardless of whether their rewards can be computed as a linear combination of a set of features or not. In order to use \sfgpi\ in practice we still need the reward features, though; our second contribution is to show that these features can be the reward functions of a set of MDPs. This reinterpretation of the problem makes it possible to combine \sfgpi\ with deep learning in a stable way. We empirically verified this claim on a complex 3D environment that requires hundreds of millions of transitions to be solved. We showed that, by turning an RL task into a supervised learning problem, \sfgpisl\ is able to provide skilful, non-trivial, policies almost instantaneously. We also showed how these policies can be used by \sfgpirl\ to learn specialised policies, which can then be added to the agent's set of skills. 
Together, these concepts can help endow an agent with the ability to build, refine, and use a set of skills while interacting with the environment.

\newpage

\section*{Acknowledgements}

The authors would like to thank Will Dabney and Hado van Hasselt for the invaluable discussions during the development of this paper. We also thank the anonymous reviewers, whose comments and suggestions helped to improve the paper considerably.

\bibliographystyle{icml2018}

\newpage
\appendix

\onecolumn
\thispagestyle{empty}

\vspace{7mm} 
\begin{center}
\vspace{7mm} 
\noindent\makebox[\textwidth]{\rule{\textwidth}{0.5pt}} \\
\vspace{3mm} 
{\bf {\LARGE  Transfer in Deep Reinforcement Learning
Using Successor Features and Generalised Policy Improvement} \\
\vspace{2mm} {\Large Supplementary Material} }
\vspace{5mm} 
\noindent\makebox[\textwidth]{\rule{\textwidth}{1.0pt}}

  {\bf Andr\'e Barreto}, \hspace{0.5mm}
  {\bf Diana Borsa}, \hspace{0.5mm}
  {\bf John Quan}, \hspace{0.5mm}
  {\bf Tom Schaul}, \hspace{0.5mm} 
  {\bf David Silver}, \hspace{0.5mm} \\
  {\bf Matteo Hessel}, \hspace{0.5mm}
  {\bf Daniel Mankowitz} \vspace{0.5mm}  
  {\bf Augustin \v{Z}\'idek}, \hspace{0.5mm}
  {\bf R\'emi Munos}, \hspace{1mm}
    \texttt{\small \{andrebarreto,borsa,johnquan,schaul,davidsilver,\\ mtthss,dmankowitz,augustinzidek,munos\}@google.com} \vspace{1.5mm} \\
  DeepMind 
  \vspace{3mm}
\end{center}

\begin{abstract}
In this supplement we give details of the theory and experiments that had to 
be left out of the main paper due to the space limit. For the convenience of the reader the statements of the theoretical results are reproduced before the respective proofs. We also report additional empirical analysis 
that could not be included in the paper.
The citations in this supplement refer to the references listed in the main paper.
\end{abstract}

\section{Proof of theoretical results}
\label{sec:new}

We restate \ctp{barreto2017successor} GPI theorem to be used as a reference in the derivations that follow.

\begin{theorem}
 \label{teo:gpi}
{\bf (Generalized Policy Improvement)} Let $\pi_1$, $\pi_2$, ..., $\pi_n$ be 
$n$ 
decision policies and 
let $\Qt^{\pi_1}$, $\Qt^{\pi_2}$, ..., $\Qt^{\pi_n}$ be approximations of their 
respective action-value functions
such that 
\begin{equation*}
\label{eq:epsilon2}
|Q^{\pi_i}(s,a) - \Qt^{\pi_i}(s,a)| \le \epsilon \, \text{ for all } s \in S, a 
\in A, \text{ and }  i \in \{1, 2, ..., n\}.
\end{equation*}
 Define
\begin{equation*}
\label{eq:pitmax2}
\pitmax(s) \in \mathop{\argmax}_a \max_i \Qt^{\pi_i}(s,a).
\end{equation*}
Then, 
\begin{equation*}
\label{eq:Qpitmax2}
\Qpitmax(s,a)  \ge \max_i Q^{\pi_i}(s,a) - \dfrac{2}{1 - \gamma} \epsilon
\end{equation*}
for any $s \in S$ and any $a \in A$,
where \Qpitmax\ is the action-value function of \pitmax.
\end{theorem}

\vspace{5mm}

\begin{lemma}
\label{teo:bound_pair2}
Let $\rdif_{ij} = \max_{s,a} \left|r_i(s,a) - r_j(s,a)\right|$ and let $\pi$ be an arbitrary policy. Then,
\begin{equation*}
\label{eq:bound_pair2}
|Q_{i}^{\pi}(s,a) - Q^{\pi}_j(s,a)| \le \dfrac{\rdif_{ij}}{1 - \gamma}.
\end{equation*}
\end{lemma}

\begin{proof}
Define $\Delta_{ij} = \max_{s,a} |Q_i^{\pi}(s,a) - Q^{\pi}_j(s,a)|$.
Then,
\begin{align}
\label{eq:der_sub_bound2}
|Q_i^{\pi}(s,a) - Q_j^{\pi}(s,a)|
\nonumber & = \left|r_i(s,a) + \gamma \sum_{s'} p(s'|s,a) Q_i^{\pi}(s',\pi(s')) 
- r_j(s,a) - \gamma \sum_{s'} p(s'|s,a) Q_j^{\pi}(s',\pi(s')) \right| \\
\nonumber & = \left|r_i(s,a) - r_j(s,a) + \gamma \sum_{s'} p(s'|s,a) 
\left(Q_i^{\pi}(s',\pi(s')) - Q_j^{\pi}(s', \pi (s')) \right) \right| \\
\nonumber & \le \left|r_i(s,a) - r_j(s,a)\right| + \gamma \sum_{s'} p(s'|s,a) 
\left| 
 Q_i^{\pi}(s',\pi(s')) - Q_j^{\pi}(s',\pi(s')) \right| \\
& \le \rdif_{ij} + \gamma \Delta_{ij}. 
\end{align}
Since~(\ref{eq:der_sub_bound2}) is valid for any $s,a \in S \times A$, we have 
shown that $\Delta_{ij} \le \rdif_{ij} + \gamma \Delta_{ij}$.
Solving for $\Delta_{ij}$ we get 
\begin{equation*}
\Delta_{ij} \le \dfrac{1}{1- \gamma} \rdif_{ij}.
\end{equation*}
\end{proof}

\begin{lemma}
\label{teo:bound_pair3}
Let $\rdif_{ij} = \max_{s,a} \left|r_i(s,a) - r_j(s,a)\right|$. Then,
\begin{equation*}
\label{eq:bound_pair3}
|Q_{i}^{\pi^*_i}(s,a) - Q^{\pi^*_j}_j(s,a)| \le \dfrac{\rdif_{ij}}{1 - \gamma}.
\end{equation*}
\end{lemma}

\begin{proof}
To simplify the notation, let $Q^i_i (s,a) \equiv Q_{i}^{\pi^*_i}(s,a)$.
Note that 
$|Q_i^i(s,a) - Q^j_j(s,a)|$ is the difference between the value functions of 
two 
MDPs with the same transition function
but potentially different rewards. 
Let $\Delta_{ij} = \max_{s,a} |Q_i^i(s,a) - Q^j_j(s,a)|$.
Then, 
{\small
\begin{align}
\label{eq:der_sub_bound1}
|Q_i^i(s,a) - Q^j_j(s,a)|
\nonumber & = \left|r_i(s,a) + \gamma \sum_{s'} p(s'|s,a) \max_b Q^i_i(s',b) 
- r_j(s,a) - \gamma \sum_{s'} p(s'|s,a) \max_b Q^j_j(s',b) \right| \\
\nonumber & = \left|r_i(s,a) - r_j(s,a) + \gamma \sum_{s'} p(s'|s,a) 
\left(\max_b Q^i_i(s',b) 
- \max_b Q^j_j(s',b) \right) \right| \\
\nonumber & \le \left|r_i(s,a) - r_j(s,a)\right| + \gamma \sum_{s'} p(s'|s,a) 
\left| 
\max_b Q^i_i(s',b) - \max_b Q^j_j(s',b) \right| \\
\nonumber & \le \left|r_i(s,a) - r_j(s,a)\right| + \gamma \sum_{s'} p(s'|s,a) 
\max_b \left| Q^i_i(s',b) - Q^j_j(s',b) \right| \\
& \le \rdif_{ij} + \gamma \Delta_{ij}. 
\end{align}
}
Since~(\ref{eq:der_sub_bound1}) is valid for any $s,a \in S \times A$, we have 
shown that 
$\Delta_{ij} \le \rdif_{ij} + \gamma \Delta_{ij}$.
Solving for $\Delta_{ij}$ we get 
\begin{equation*}
\Delta_{ij} \le \dfrac{1}{1- \gamma} \rdif_{ij}.
\end{equation*}
\end{proof}

\setcounter{proposition}{0}
\begin{proposition}
Let $M \in \MM$ and let $Q^{\pi_j^{*}}_i$ be the action-value function of an optimal 
policy of $M_j \in \MM$ when executed in $M_i \in \MM$. Given approximations 
$\{\Qt^{\pi_1}_i, \Qt^{\pi_2}_i, ..., \Qt^{\pi_n}_i\}$ such that 
$
 \left|Q^{\pi_j^{*}}_i(s, a) - \Qt^{\pi_j}_i(s,a) \right| \le \epsilon
$
for all $s \in \S$, $a \in \A$, and $j \in \{1, 2, ..., n\}$,
let 
\begin{equation*}
\pihmax(s) \in \argmax_a {\max}_j \Qt^{\pi_j}_i (s,a).
\end{equation*}
Then,
\begin{equation*}
\|Q^{*} - Q^{\pi}\|_{\infty} 
\le \dfrac{2}{1-\gamma} \left( 
\| r - r_i \|_{\infty} +
\mathop{\min}_j \|r_i - r_j \|_{\infty} + \epsilon
\right),
\end{equation*}
where $Q^*$ is the optimal value function of $M$, $Q^\pi$ is the value function of 
$\pi$ in $M$, and $\|f - g\|_\infty = \max_{s,a}|f(s,a) - g(s,a)|$.
\end{proposition}

\begin{proof}
The result is a direct application of Theorem~\ref{teo:gpi} and Lemmas~\ref{teo:bound_pair2} and~\ref{teo:bound_pair3}. Let $\pi^*$ be an optimal value function of $M$. Then,
\begin{equation*}
\begin{array}{llr}
Q^{*}(s,a) - Q^{\pi}(s,a) 
& =  Q^{\pi^{*}}(s,a) - Q^{\pi}(s,a) \\
& =  Q^{\pi^{*}}(s,a) - Q^{\pi^*_i}_i 
+ Q^{\pi^*_i}_i - Q^{\pi}(s,a) \\
& =  Q^{\pi^{*}}(s,a) - Q^{\pi^*_i}_i 
+ Q^{\pi^*_i}_i - Q^{\pi}_i
+ Q^{\pi}_i - Q^{\pi}(s,a) \\
& \le  |Q^{\pi^{*}}(s,a) - Q^{\pi^*_i}_i| 
+ Q^{\pi^*_i}_i - Q^{\pi}_i
+ |Q^{\pi}_i - Q^{\pi}(s,a)| \\
\end{array}
\end{equation*}
From Lemma~\ref{teo:bound_pair3}, we know that
\begin{equation*}
|Q^{\pi^{*}}(s,a) - Q^{\pi^*_i}_i| 
\le \dfrac{\max_{s,a} \left|r(s,a) - r_i(s,a)\right|}{1 - \gamma}.
\end{equation*}
From Theorem~\ref{teo:gpi} we know that, for any $j \in \{1, 2, ..., n\}$, we have
\begin{equation}
\label{eq:app_gpi}
\begin{array}{llr}
Q_i^{\pi_i^*}(s,a) - Q^{\pi}_i(s,a)  
& \le Q_i^{\pi_i^*}(s,a) - Q^{\pi_j^*}_i(s,a)  + \dfrac{2}{1 - \gamma} \epsilon 
& \text{(Theorem~\ref{teo:gpi})} \\
& = Q_i^{\pi_i^*}(s,a) - Q^{\pi_j^*}_j(s,a) + Q^{\pi_j^*}_j(s,a) - Q^{\pi_j^*}_i(s,a)  + \dfrac{2}{1 - \gamma} \epsilon \\
& \le |Q_i^{\pi_i^*}(s,a) - Q^{\pi_j^*}_j(s,a)| + |Q^{\pi_j^*}_j(s,a) - Q^{\pi_j^*}_i(s,a)|  + \dfrac{2}{1 - \gamma} \epsilon \\
& \le \dfrac{2}{1-\gamma} \max_{s,a} \left|r_i(s,a) - r_j(s,a)\right| + \dfrac{2}{1 - \gamma} \epsilon 
& \text{(Lemmas~\ref{teo:bound_pair2} and~\ref{teo:bound_pair3})}. \\
\end{array}
\end{equation}
Finally, from Lemma~\ref{teo:bound_pair2}, we know that
\begin{equation*}
|Q^{\pi}_i - Q^{\pi}(s,a)| 
\le \dfrac{\max_{s,a} \left|r(s,a) - r_i(s,a)\right|}{1 - \gamma}.
\end{equation*}
\end{proof}

\section{Details of the experiments}
\label{sec:exp_details}

In this section we give details of the experiments that had to 
be left out of the main paper due to the space limit.

\subsection{Agents's architecture}

The CNN used in Figure~\ref{fig:architecture} is identical to that used by \ctp{mnih2015human} DQN. The CNN outputs a $256$-dimensional vector that serves as the LSTM state. As shown in Figure~\ref{fig:architecture}, the LSTM also receives the previous action of the agent as an input. The output of the LSTM is a vector of dimension  $256$, which in the paper we call the state signal $\tilde{s}$. The vector $\tilde{s}$ is the input of the $D+1$ MLPs used to compute \tvphi\ and $\tvpsi^{\pi_i}$. These MLPs have $100$ $\tanh$ hidden units and an output of dimension $D \times |\A|$---that is, for each action $a \in \A$ the MLP outputs a $D$-dimensional vector representing either \tvphi\ or one of the $\tvpsi^{\pi_i}$. These $D$-dimensional vectors are then multiplied by $\wt$, leading to a $(D+1) \times |\A|$ output representing \rt\ and $\Qt^{\pi_i}$.

\subsection{Agents's training}

The losses shown in lines~\ref{it:learn_w} and~\ref{it:td} of Algorithm~\ref{alg:sfgpi} and in lines~\ref{it:learn_r} and~\ref{it:learn_q} of Algorithm~\ref{alg:sfgpi_basis} were minimised using the RMSProp method, a variation of the well-known back-propagation algorithm. As parameters of RMSProp we adopted a fixed decay rate of $0.99$ and $\epsilon = 0.01$. For all algorithms we tried at least two values for the learning rate: $0.01$ and $0.001$. For the baselines ``$DQ(\lambda)$ fine tuning'' and ``$DQ(\lambda)$ from scratch'' we also tried a learning rate of $0.005$. The results shown in the paper are those associated with the best final performance of each algorithm. 

As mentioned in the paper, the agents's training was carried out using the IMPALA architecture~\cp{espehold2018impala}. 
In IMPALA the agent is conceptually divided in two groups: ``actors'', which interact with the environment in parallel collecting trajectories and adding them to a queue, and a ``learner'', which pulls trajectories from the queue and uses them to apply the updates. 
On the learner side, we adopted a simplified version of IMPALA that uses $Q(\lambda)$ as the RL algorithm ({\sl i.e.}, no parametric representation of policies nor off-policy corrections). For the distributed collection of data we used $50$ actors per task. Each actor gathered trajectories of length $20$ that were then added to the common queue. The collection of data followed an $\epsilon$-greedy policy with a decaying $\epsilon$. Specifically, the value of $\epsilon$ started at $0.5$ and decayed linearly to $0.05$ in $10^6$ steps. The results shown in the paper correspond to the performance of the $\epsilon$-greedy policy (that is, they \emph{include} exploratory actions of the agents).

 For the results with \sfgpirl, in addition to the loss induced by equation~(\ref{eq:bellman_psi}), minimised in line~\ref{it:learn_psi} of Algorithm~\ref{alg:sfgpi}, we also used a standard $Q(\lambda)$ loss---that is, the gradients associated with both losses were combined through a weighted sum and then used to update $\params_\psi$. The weights for the standard $Q(\lambda)$ loss and the loss computed in line~\ref{it:td} of Algorithm~\ref{alg:sfgpi} were $1$ and $0.1$, respectively. Using the standard $Q(\lambda)$ loss seems to stabilise the learning of $\tvpsi^{\pi_{n+1}}$; in this case~(\ref{eq:bellman_psi}) can be seen as a constraint for the standard RL optimisation. Obviously, if we want to add $\tvpsi^{\pi_{n+1}}$ to \tPsi, we have to make sure that the SF semantics is preserved---that is, the result of the combined updates approximately satisfies (\ref{eq:bellman_psi}). We confirmed this fact by monitoring the loss computed in line~\ref{it:td} of Algorithm~\ref{alg:sfgpi}. Figure~\ref{fig:psi_loss} shows the average of this loss computed over $10$ runs of \sfgpirl\ on all test tasks; as shown in the figure, the loss is indeed minimised, which implies that the resulting $\tvpsi^{\pi_{n+1}}$ are valid SFs that can be safely added to \tPsi. 
 
\begin{figure}[h]
\centering
\includegraphics[scale=0.6]{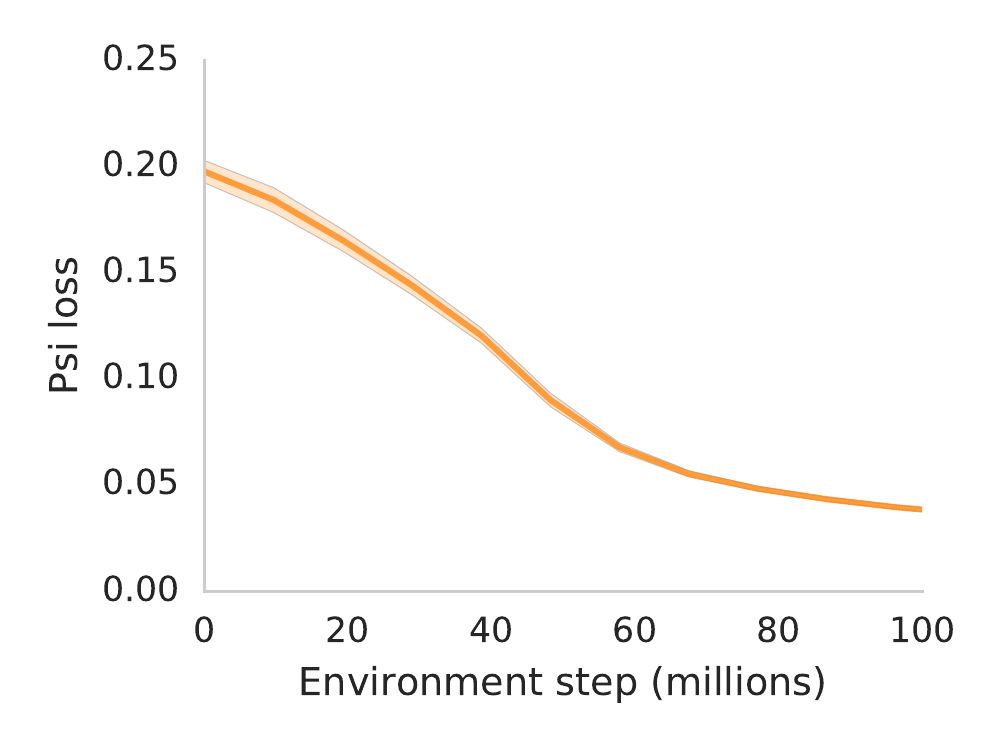}
\caption{Loss in the approximation of $\tvpsi^{\pi_{n+1}}$ (line~\ref{it:td} of Algorithm~\ref{alg:sfgpi}). Shaded region represents one standard deviation over $10$ runs on all test tasks. \label{fig:psi_loss} \vspace{5mm} }
\end{figure}

\subsection{Environment}

The room used in the environment and illustrated in Figure~\ref{fig:observation} was of size $13 \times 13$~\cp{beattie2016deepmind}.
The observations $o_t$ were an $84 \times 84$ image with pixels re-scaled to the interval $[0,1]$. The action space \A\ contains $8$ actions: move forward, move backwards, strafe left, strafe right, look left, look right, look left and move forward, and look right and move forward. Each action was repeated for $4$ frames, that is, the agent was allowed to choose an action at every $4$ observations (we note that the ``environment steps'' shown in the plots refer to actual observations, not the number of decisions made by the agent). 

\section{Additional results}
\label{sec:add_resuts}

In our experiments we defined a set of $9$ test tasks in order to cover reasonably well three qualitatively distinct combinations of rewards: only positive rewards, only negative rewards, and mixed rewards. Figure~\ref{fig:add_results} shows the results of \sfgpisl\ and the baselines on the test tasks that could not be included in the paper due to the space limit. 

\begin{figure*}
\centering
\newcommand{\scl}{0.7}
\subfigure[Task \tasksix\ \label{fig:test_6}]{
\includegraphics[scale=\scl]{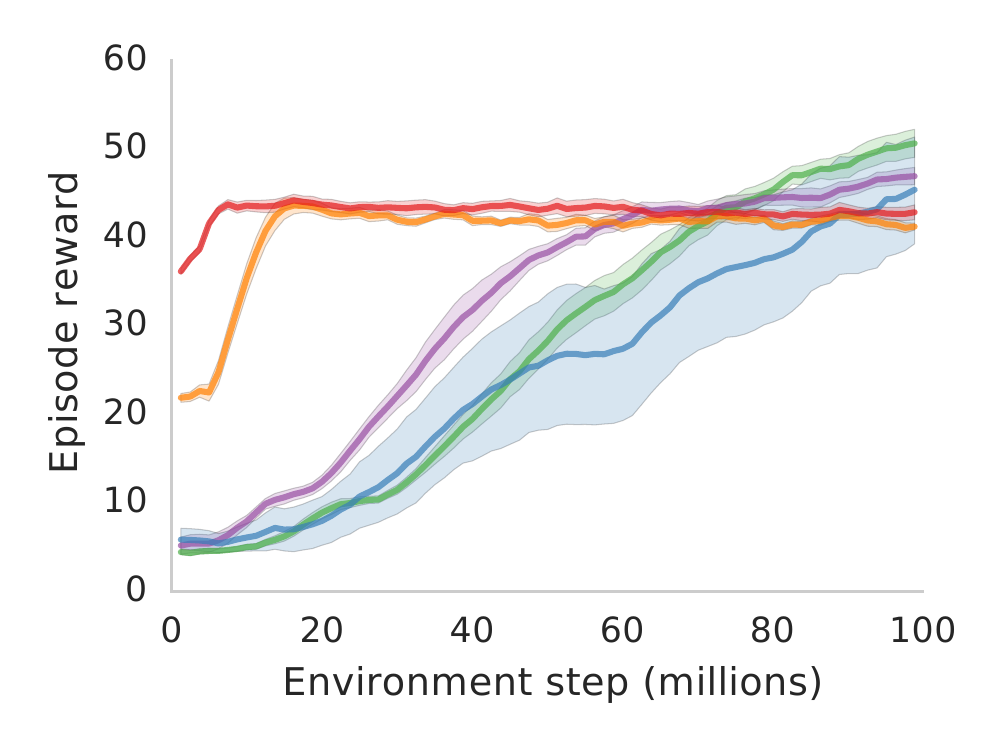}
}
 \subfigure[Task \taskseven\ ]{
 \includegraphics[scale=\scl]{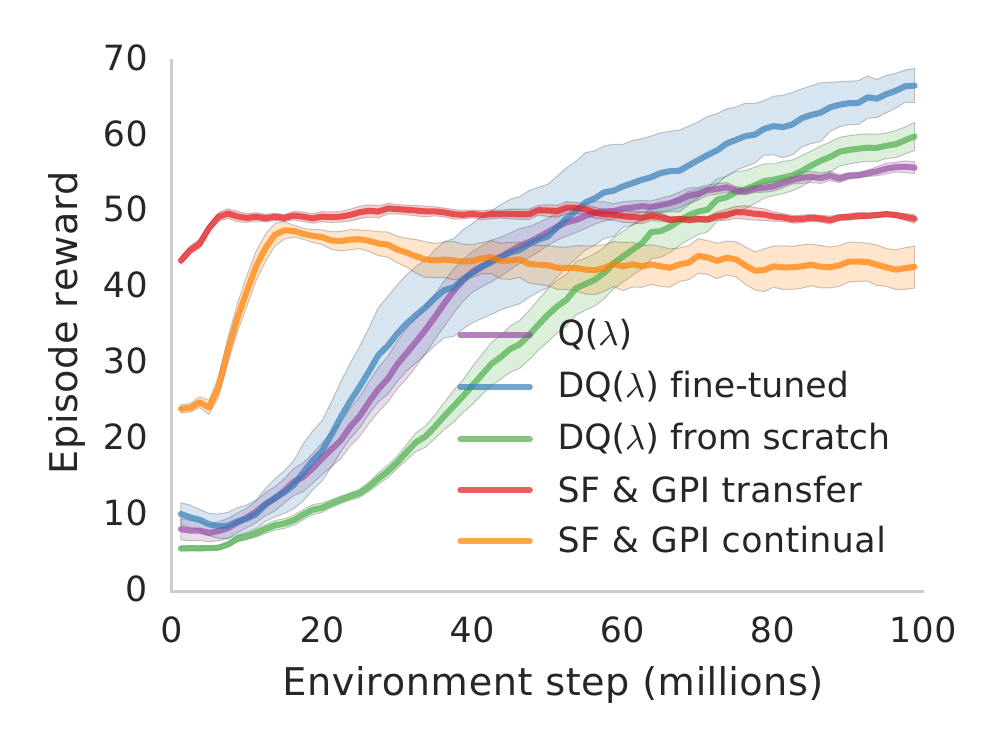}
 }
 
\subfigure[Task \taskeight\ ]{
\includegraphics[scale=\scl]{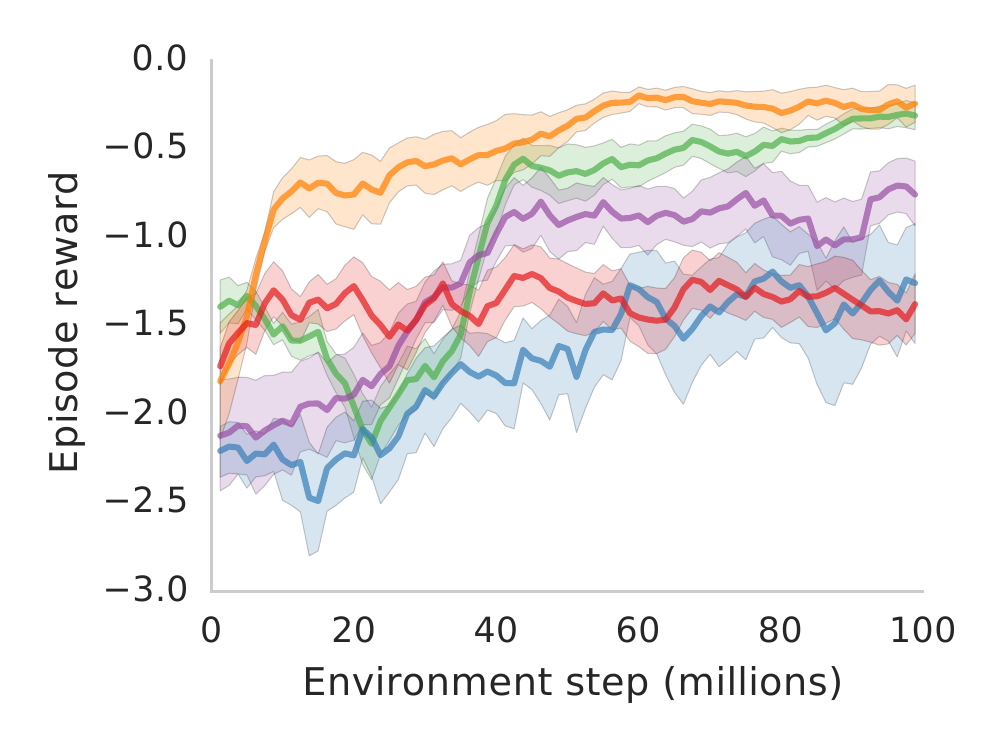}
}
\subfigure[Task \taskeleven\ ]{
\includegraphics[scale=\scl]{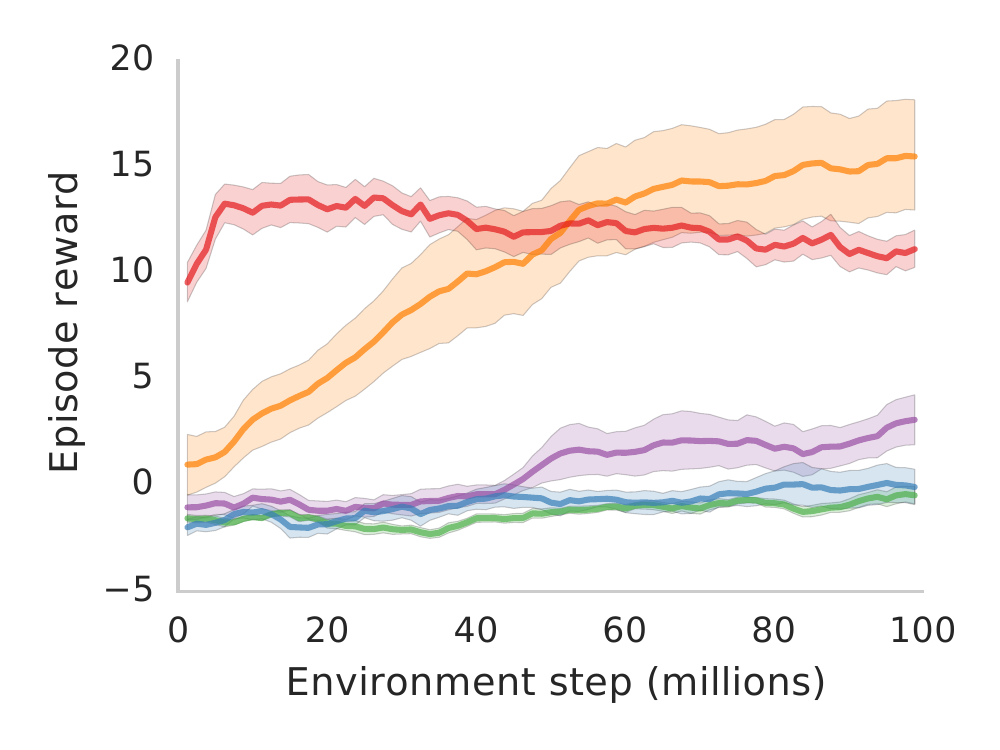}
}
\caption{Average reward per episode on test tasks not shown in the main paper. The $x$ axes have different scales because the amount of reward available changes across tasks. Shaded regions are one standard deviation over $10$ runs. \label{fig:add_results}
}
\end{figure*}

As discussed in the paper, ideally our agent should rely on the GPI policy when useful but also be able to learn and use a specialised policy otherwise. Figures~\ref{fig:add_analysis}, \ref{fig:add_analysis2} and ~\ref{fig:add_analysis3} show that this is possible with \sfgpirl.
Looking at Figure~\ref{fig:add_analysis} we see that when the test task only has positive rewards the performances of \sfgpisl\ and \sfgpirl\ are virtually the same. This makes sense, since in this case alternating between the policies $\pi_i$ learned on \MMtr\ should lead to good performance. Although initially the specialised policy \pitest\ does get selected by GPI a few times, eventually the policies $\pi_i$ largely dominate. The figure also corroborates the hypothesis that GPI is in general not computing a trivial policy, since even after settling on the policies $\pi_i$ it keeps alternating between them.

Interestingly, when we look at the test tasks with negative rewards this pattern is no longer observed. As shown in Figures~\ref{fig:add_analysis2} and~\ref{fig:add_analysis3}, in this case \sfgpirl\  eventually outperforms \sfgpisl---which is not surprising. Looking at the frequency at which policies are selected by GPI, we observe the opposite trend as before: now the policy \pitest\ steadily becomes the preferred one. The fact that a specialised policy is learned and eventually dominates is reassuring, as it indicates that \pitest\ will contribute to the repository of skills available to the agent when added to~\tPsi.    

As discussed in the main paper, one of the highlights of the proposed algorithm is the fact that it can use any set of base tasks and enable transfer to related, unseen, tasks. To empirically verify this claim, we tested several sets of base tasks. We used as a reference our ``canonical'' set of base tasks \MMtr\ that spans the environment $\MM$. We further validated our approach on a linearly-depend set of base tasks $\MMtr'$ that does not span the set of (test) tasks we are interested in---these results are shown in Figure~\ref{fig:basis_comparison}. In addition to these, we now present experiments with a third set of base tasks that does span the set of tasks but does not include any of the canonical tasks: $\MMtr'' = \{ {\tt (1,\sm0.1,\sm0.1,\sm0.1)},{\tt (\sm0.1, 1, \sm0.1,\sm0.1)},{\tt (\sm0.1,\sm0.1, 1, \sm0.1)},{\tt (\sm0.1, \sm0.1, \sm0.1, 1)} \}$. In Figure \ref{fig:appendix_basis_comparison} we report results obtained by \sfgpisl, using the three sets of base tasks described, on the $9$ previously-introduced test tasks. We can see that all sets of base tasks lead to satisfactory transfer, but the performance of the transferred policy can vary significantly depending on the relation between the base tasks used and the test task. This is again an illustration of the distinction between the ``reward basis'' and the ``behaviour basis'' discussed in Section~\ref{sec:results}.

\begin{figure*}
\centering
\newcommand{\scl}{0.7}
\newcommand{\scll}{0.15}

\subfigure[Task \taskfive\ ]{
 \begin{tabular}{r}
 \includegraphics[scale=\scl]{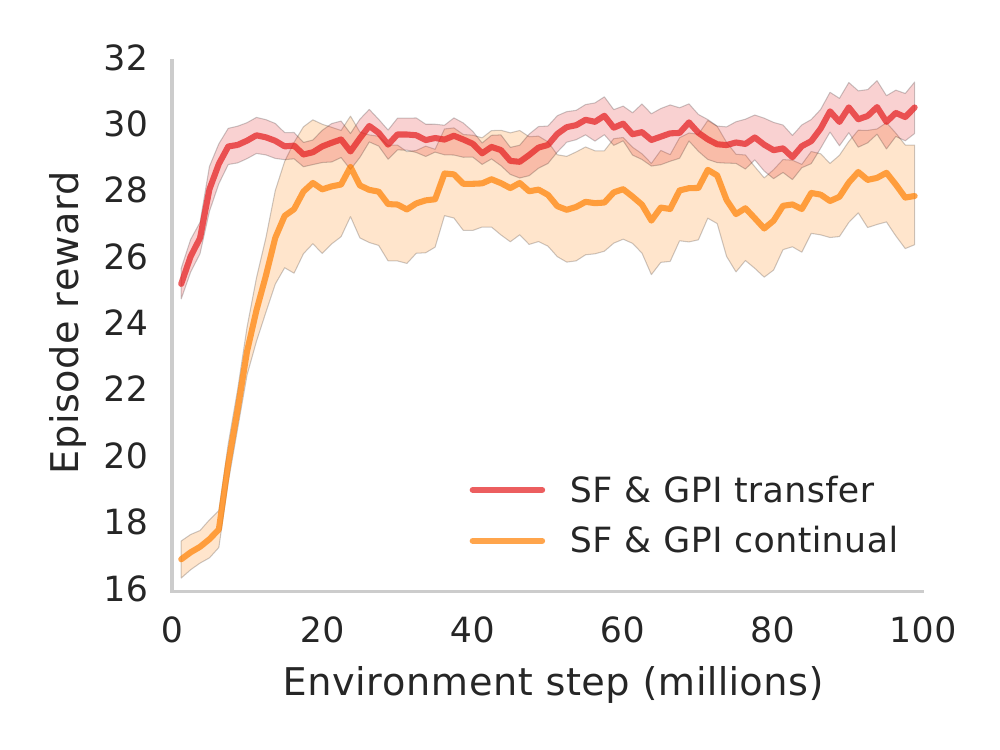} \\
 \includegraphics[scale=\scll]{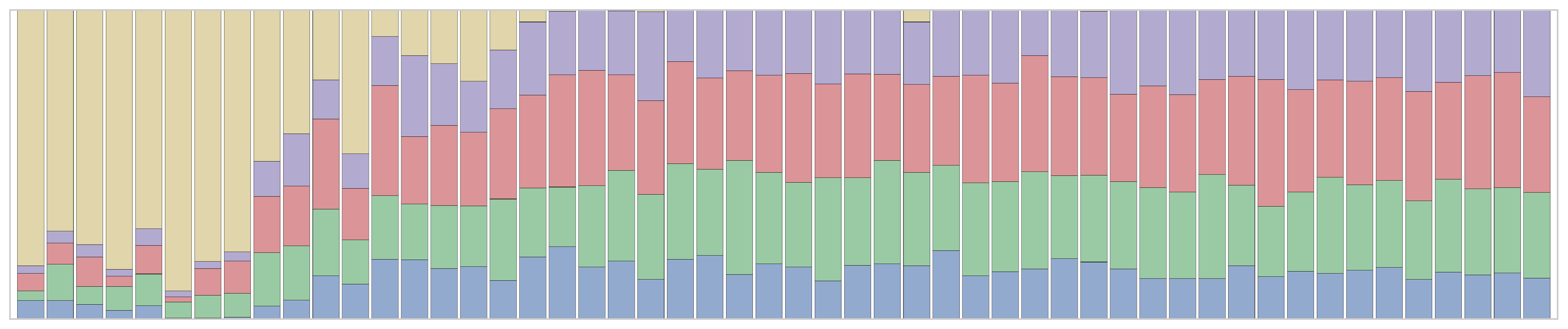}
 \end{tabular}
 }
\subfigure[Task \tasksix\ ]{
 \begin{tabular}{r}
 \includegraphics[scale=\scl]{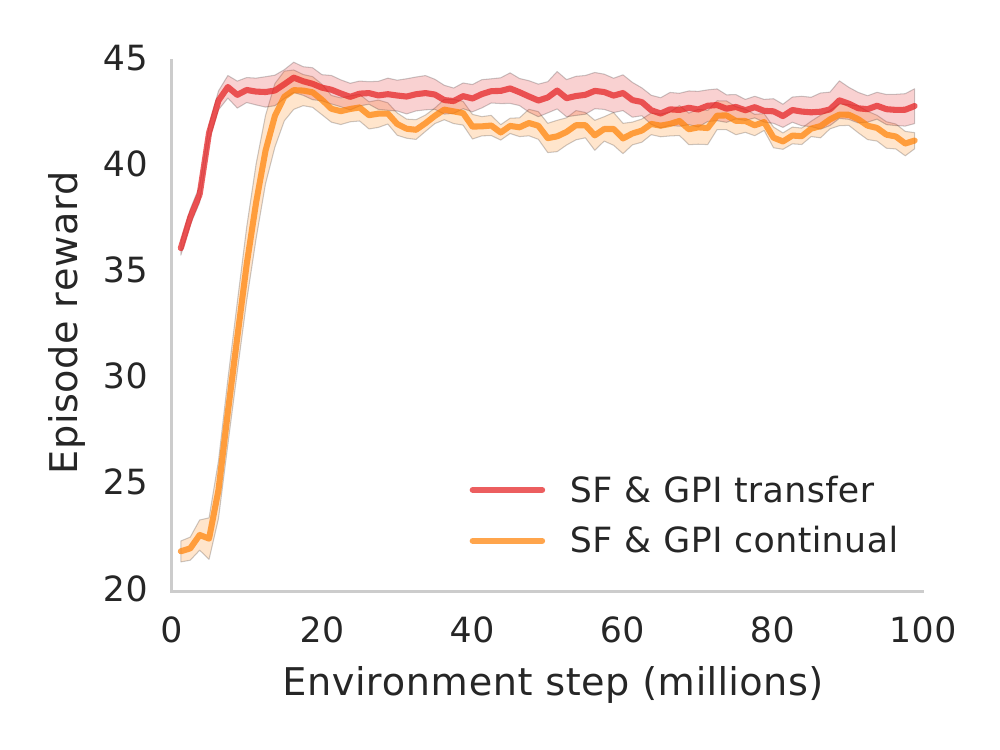} \\
 \includegraphics[scale=\scll]{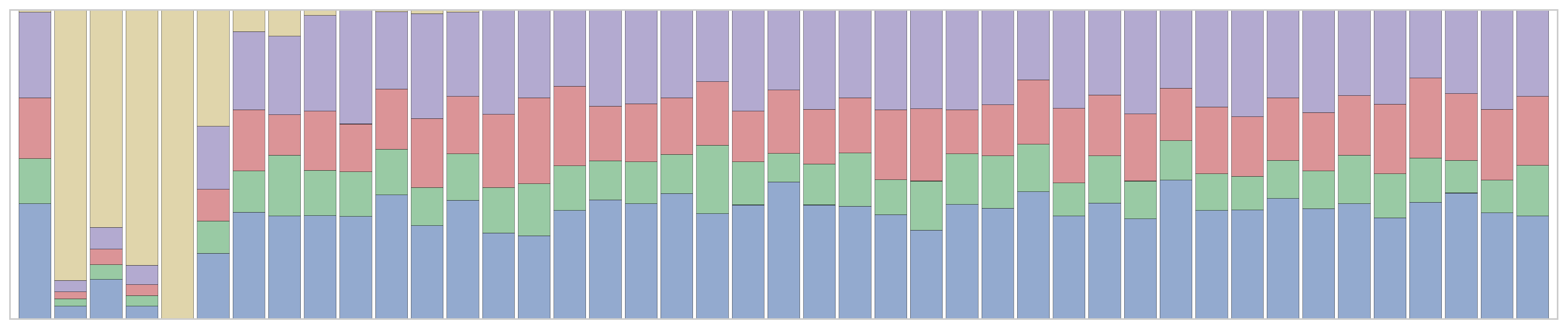}
 \end{tabular}
 }

\subfigure[Task \taskseven\ ]{
 \begin{tabular}{r}
 \includegraphics[scale=\scl]{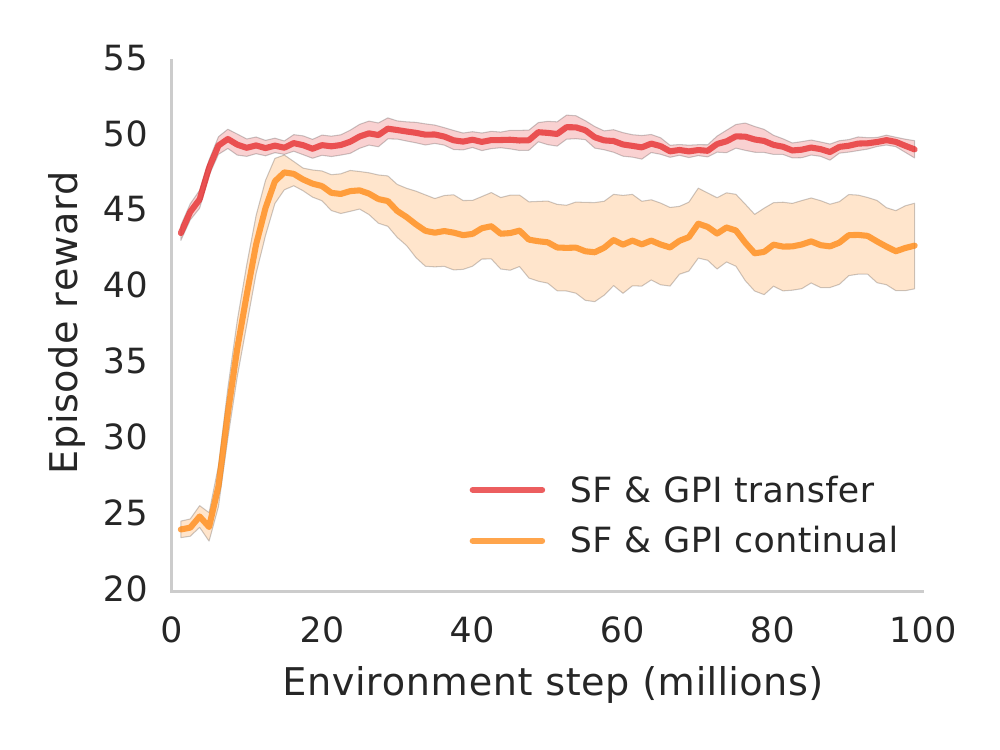} \\
 \includegraphics[scale=\scll]{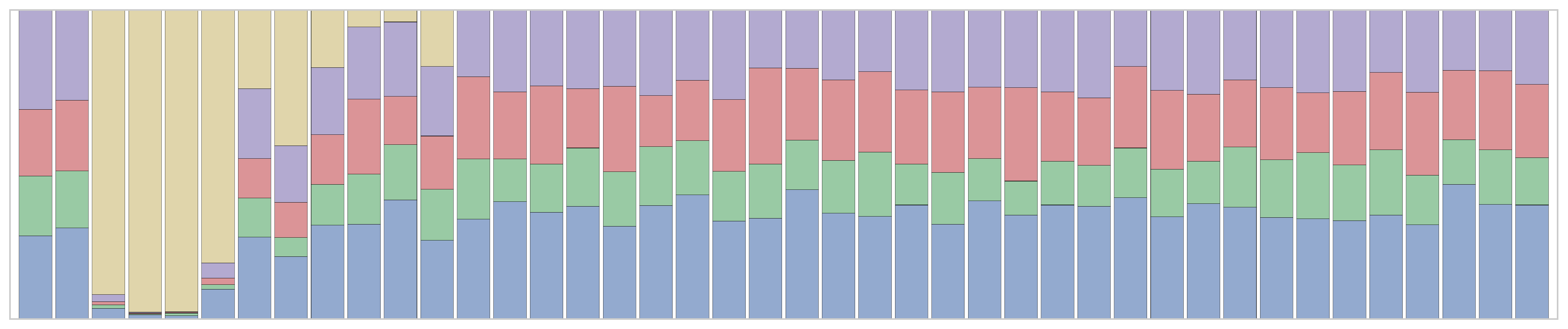}
 \end{tabular}
 }

\caption{{\bf Top figures}: Comparison between \sfgpisl\ and \sfgpirl. 
Shaded regions are one standard deviation over $10$ runs. All the runs of \sfgpisl\ and \sfgpirl\ used the same basis \tPsi. {\bf Bottom figures}: Coloured bar segments represent the frequency at which the policies $\pi_i$ were selected by GPI in one run of \sfgpirl, with each colour associated with a specific policy. The policy \pitest\ specialised to the task is represented in light yellow. \label{fig:add_analysis}}
\end{figure*}

\begin{figure*}
\centering
\newcommand{\scl}{0.7}
\newcommand{\scll}{0.15}

\subfigure[Task \taskeight\ ]{
 \begin{tabular}{r}
\includegraphics[scale=\scl]{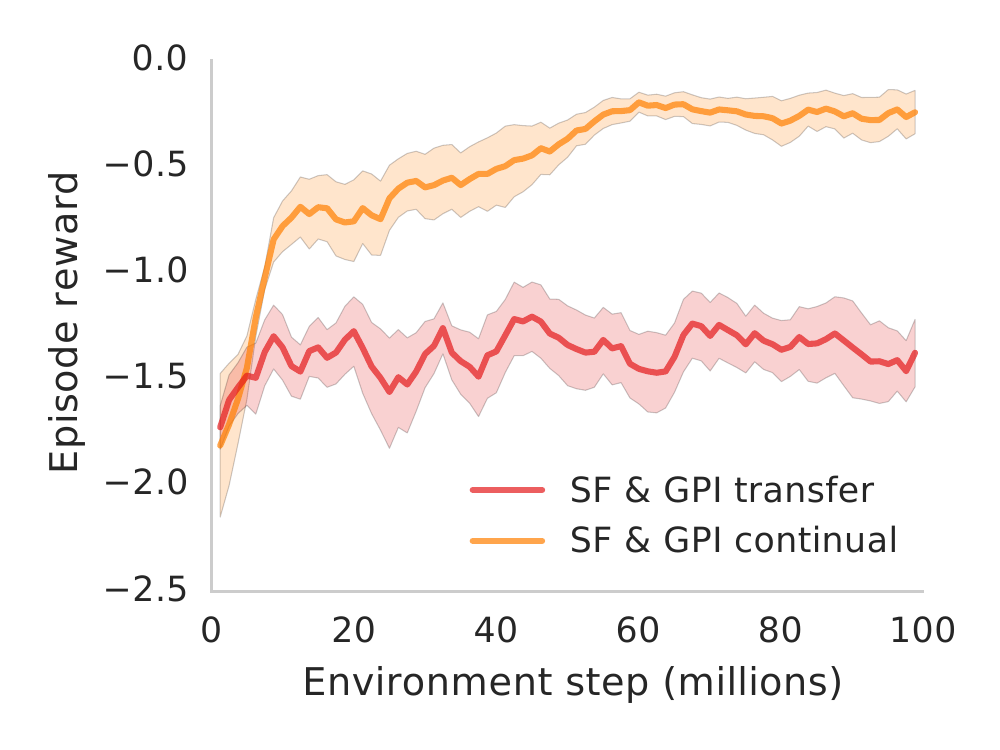} \\
\includegraphics[scale=\scll]{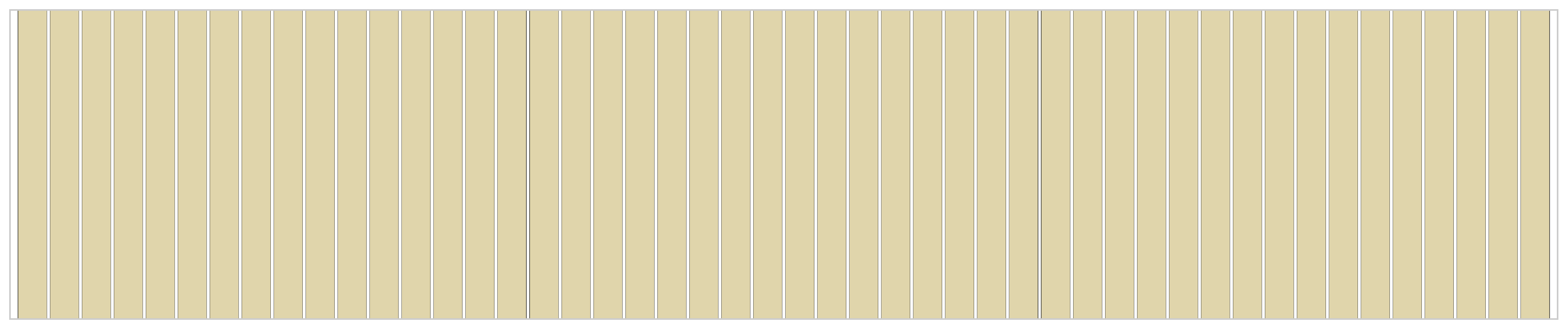}
 \end{tabular}
}
\subfigure[Task \tasknine\ ]{
 \begin{tabular}{r}
\includegraphics[scale=\scl]{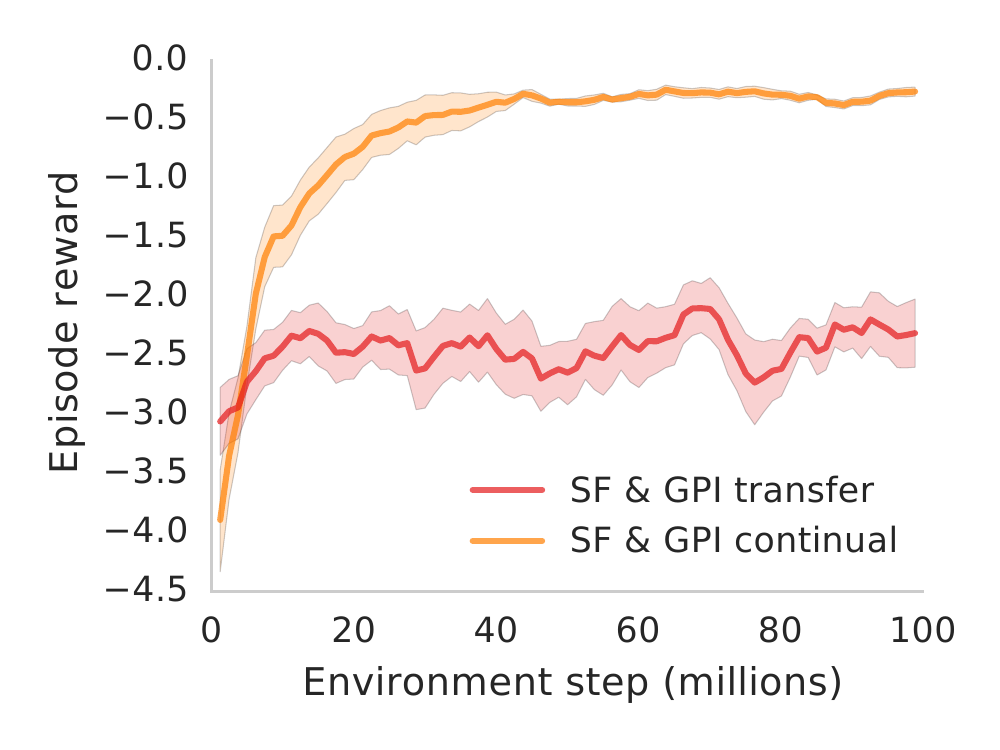} \\
\includegraphics[scale=\scll]{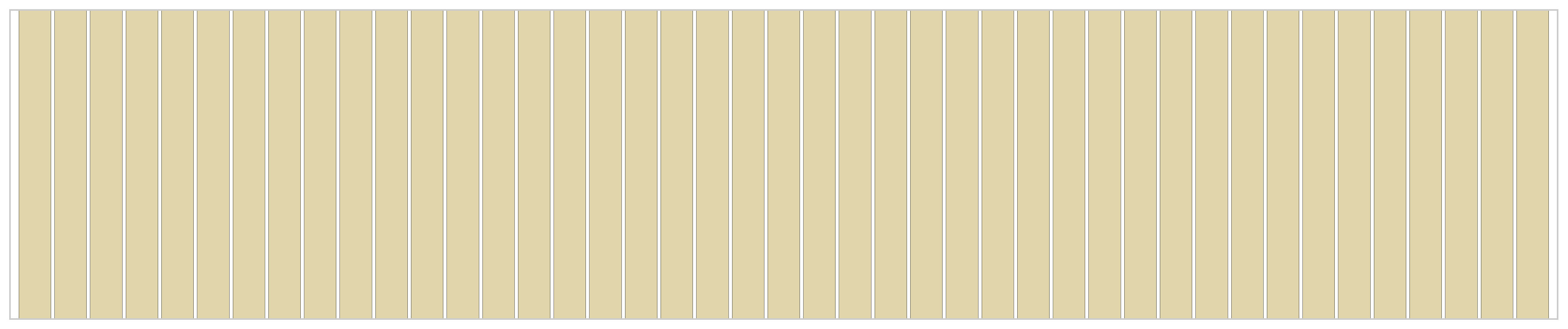}
 \end{tabular}
}

\subfigure[Task \taskten\ ]{
 \begin{tabular}{r}
\includegraphics[scale=\scl]{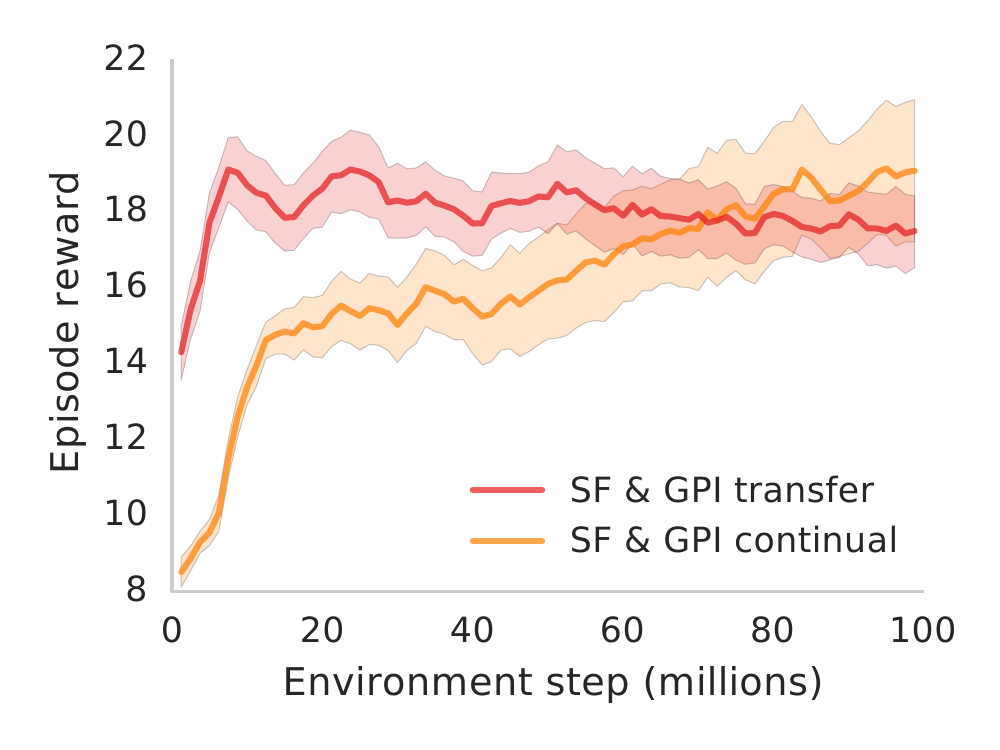} \\
\includegraphics[scale=\scll]{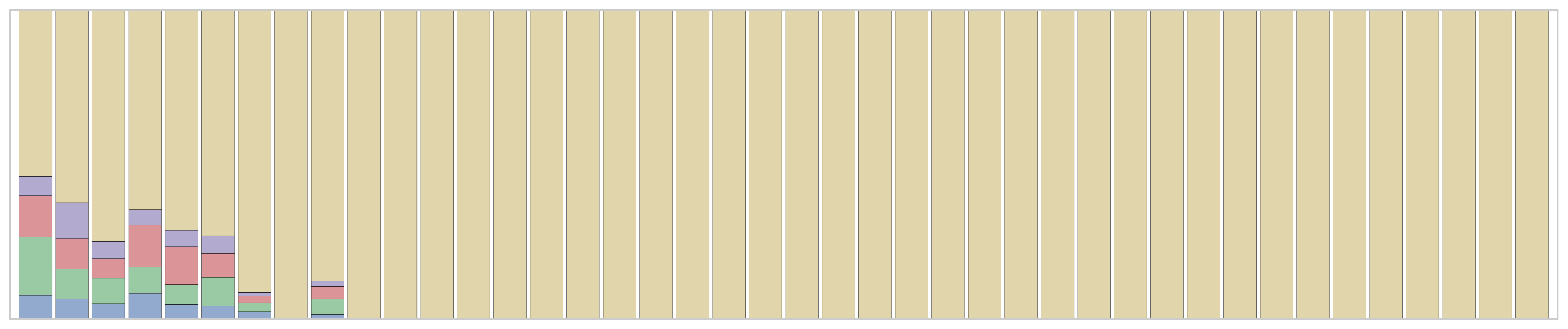}
 \end{tabular}
}

\caption{{\bf Top figures}: Comparison between \sfgpisl\ and \sfgpirl. 
Shaded regions are one standard deviation over $10$ runs. All the runs of \sfgpisl\ and \sfgpirl\ used the same basis \tPsi. {\bf Bottom figures}: Coloured bar segments represent the frequency at which the policies $\pi_i$ were selected by GPI in one run of \sfgpirl, with each colour associated with a specific policy. The policy \pitest\ specialised to the task is represented in light yellow. \label{fig:add_analysis2}}
\end{figure*}

\begin{figure*}
\centering
\newcommand{\scl}{0.7}
\newcommand{\scll}{0.15}

\subfigure[Task \taskeleven\ ]{
 \begin{tabular}{r}
\includegraphics[scale=\scl]{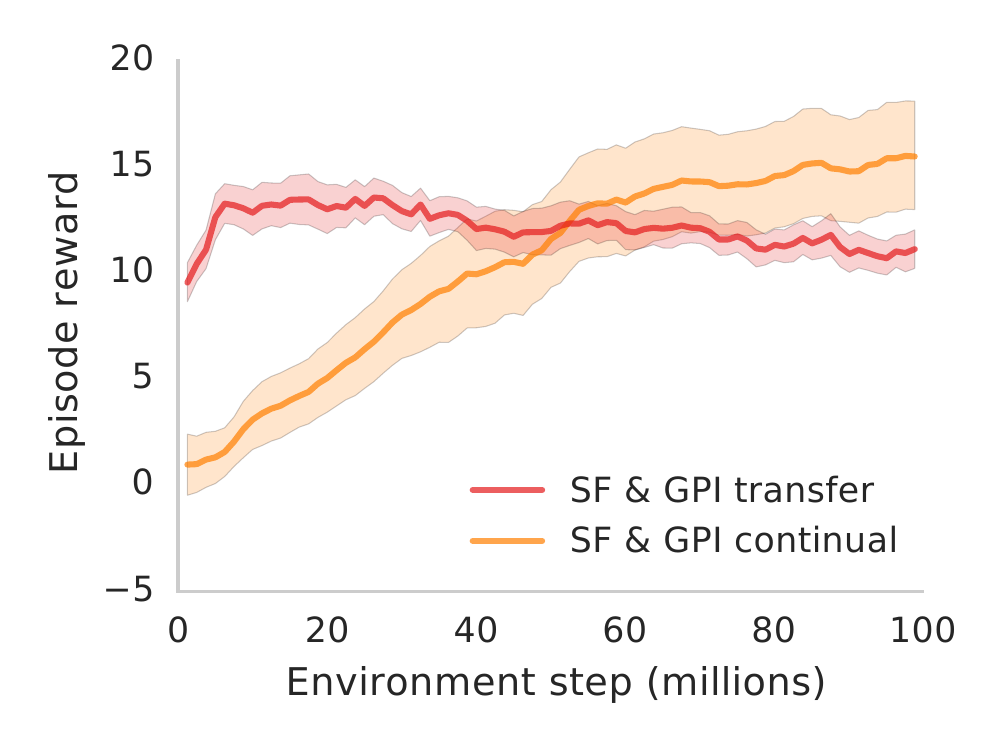} \\
\includegraphics[scale=\scll]{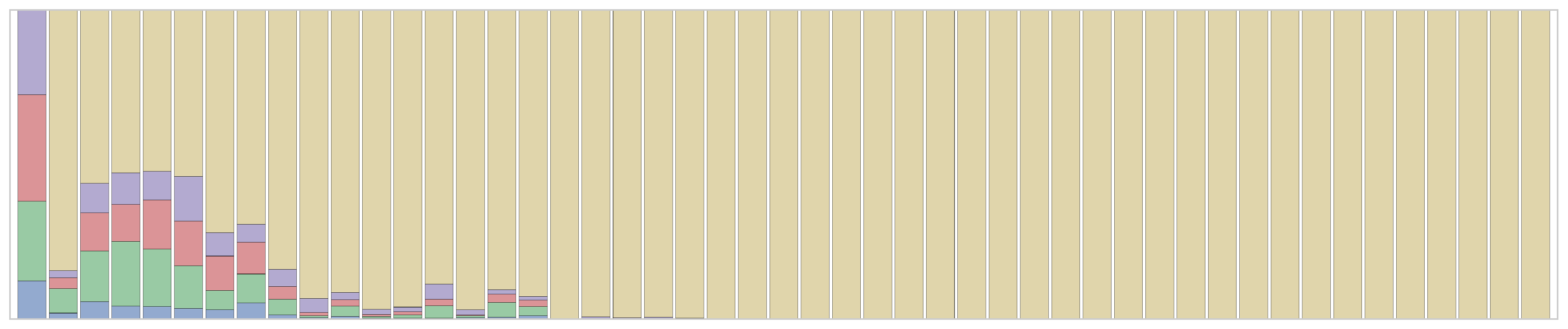}
 \end{tabular}
}
 \subfigure[Task \tasktwelve\ ]{
 \begin{tabular}{r}
 \includegraphics[scale=\scl]{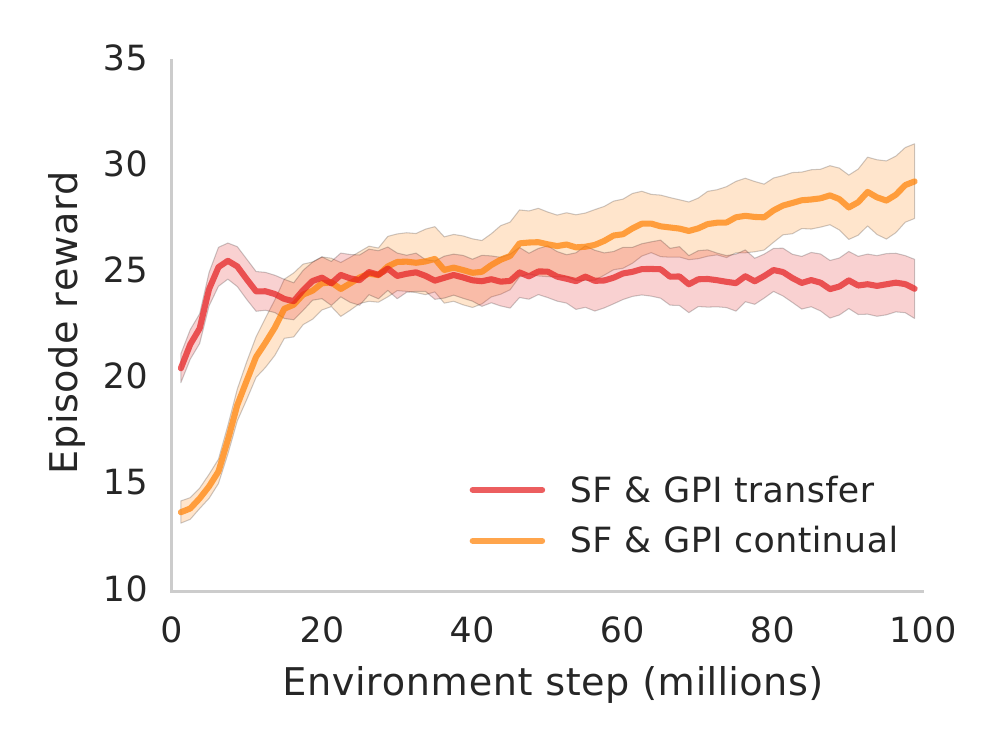} \\
 \includegraphics[scale=\scll]{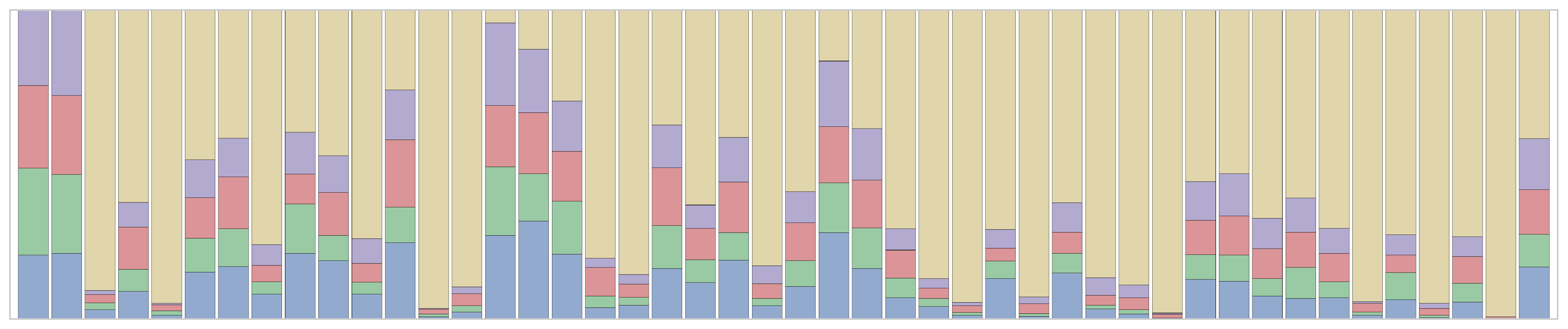}
 \end{tabular}
 }
 
 \subfigure[Task \taskthirteen\ ]{
 \begin{tabular}{r}
 \includegraphics[scale=\scl]{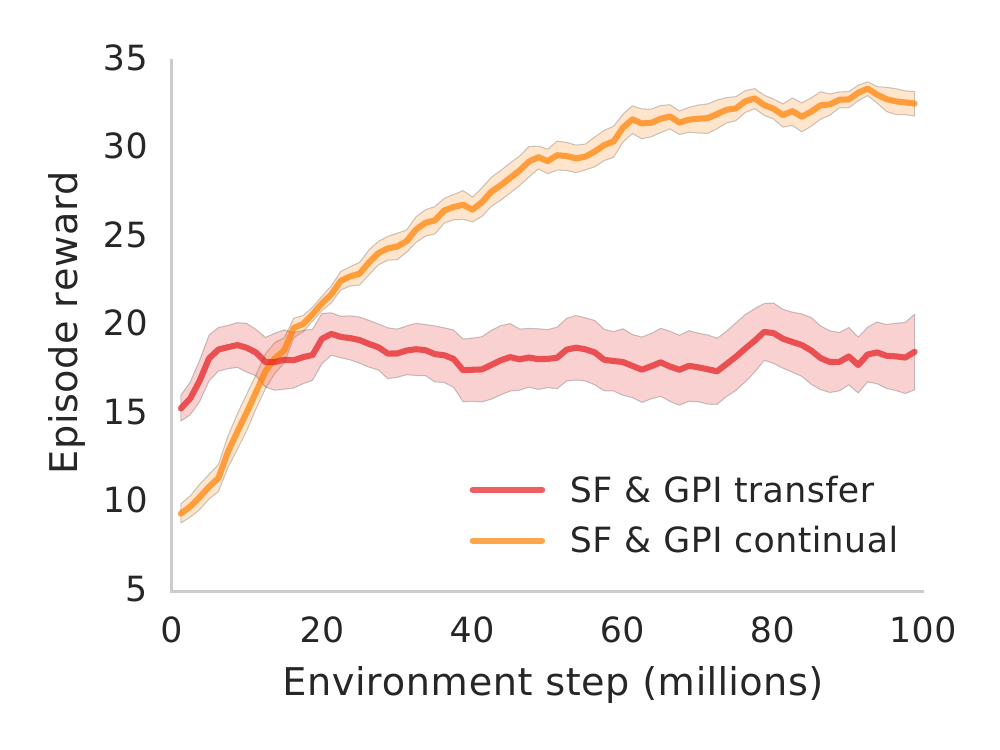} \\
 \includegraphics[scale=\scll]{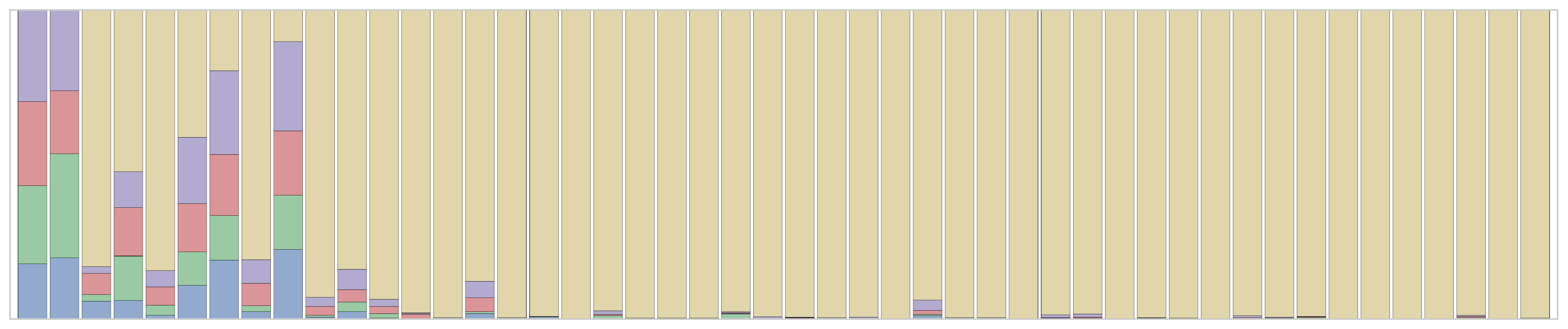}
 \end{tabular}
 }

 \caption{{\bf Top figures}: Comparison between \sfgpisl\ and \sfgpirl. 
Shaded regions are one standard deviation over $10$ runs. All the runs of \sfgpisl\ and \sfgpirl\ used the same basis \tPsi. {\bf Bottom figures}: Coloured bar segments represent the frequency at which the policies $\pi_i$ were selected by GPI in one run of \sfgpirl, with each colour associated with a specific policy. The policy \pitest\ specialised to the task is represented in light yellow. \label{fig:add_analysis3}}
\end{figure*}

\begin{figure}
\begin{center}
\includegraphics[scale=0.5]{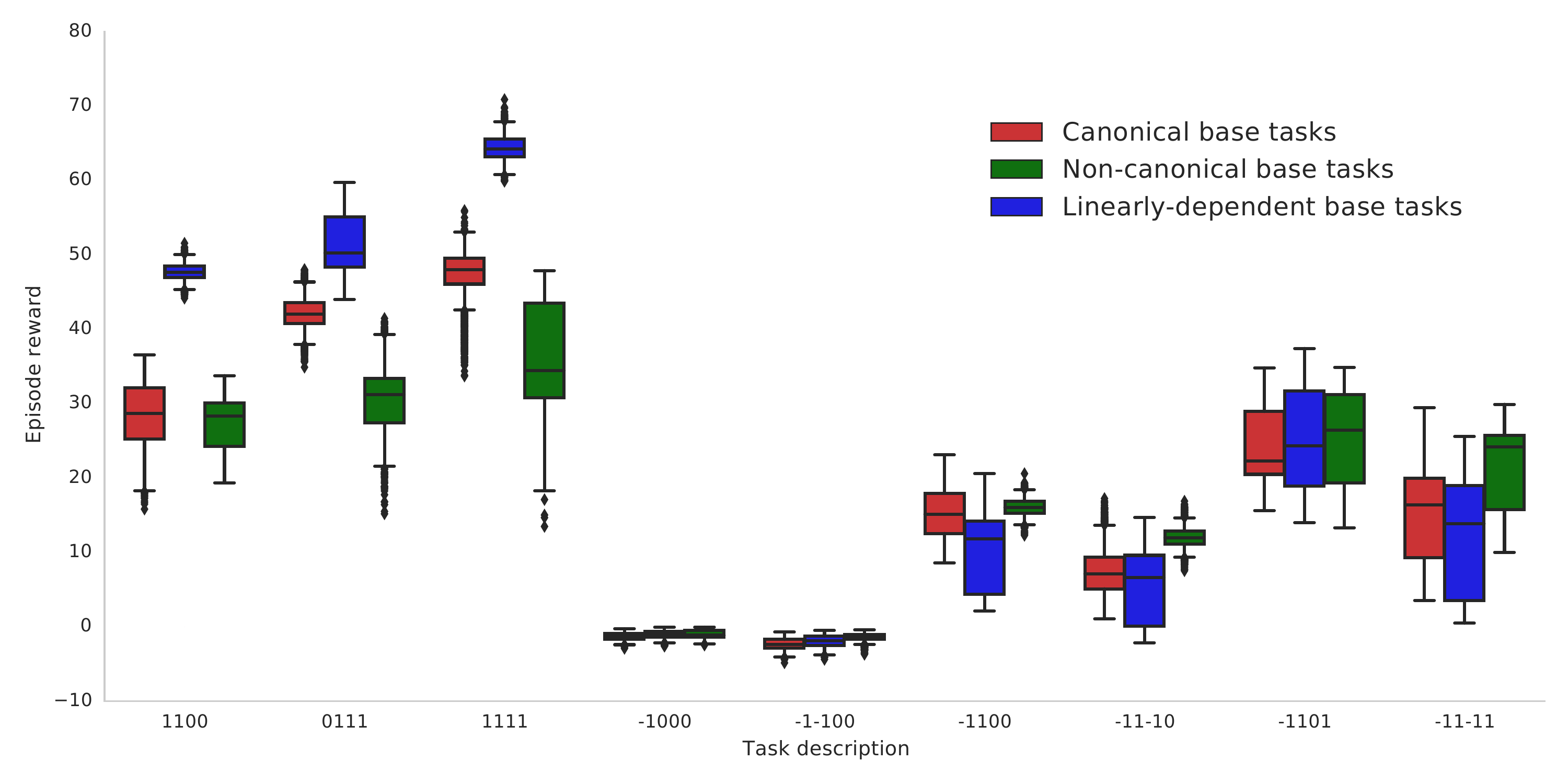}
\end{center}
\vspace{-2mm} 
\caption{Performance of \sfgpisl\ using base tasks $\MMtr$, $\MMtr'$ and $\hat{\mathcal{M}}''$, on the $9$ test tasks. The box plots summarise the distribution of the rewards received per episode between $50$ and $100$ million steps of learning. 
\label{fig:appendix_basis_comparison} }
\vspace{-2mm}
\end{figure}

\end{document}